\pdfoutput=1
\documentclass[final]{colt2021} 

\usepackage{notation} 
\let\ceil\undefined   
\let\abs\undefined    
\DeclareMathOperator{\ex}{\mathbb E}
\usepackage[utf8]{inputenc} 
\usepackage[T1]{fontenc}    
\usepackage{hyperref}       
\usepackage{booktabs}       
\usepackage{microtype}      

\usepackage{tikz}

\usepackage[shortlabels]{enumitem}

\usepackage{bm}
\usepackage{bbm}
\usepackage{textcomp}

\DeclareMathOperator*{\argmax}{\arg\!max}

\let\inf\undefined
\let\min\undefined
\let\max\undefined
\DeclareMathOperator*{\inf}{inf\vphantom{p}} 
\DeclareMathOperator*{\min}{min\vphantom{p}} 
\DeclareMathOperator*{\max}{max\vphantom{p}} 
\newcommand{\KL}[0]{\operatorname{KL}}
\newcommand{\KLinf}[0]{\operatorname{KL_{inf}}}

\newcommand{\KLinfR}[1]{\operatorname{KL}_{\operatorname{inf}}^{#1}}

\newcommand{\Supp}[0]{\operatorname{Supp}}

\newcommand{\lrset}[1]{\left\{{#1}\right\}}
\newcommand{\lrp}[1]{\left({#1}\right)}
\newcommand{\ceil}[1]{\left\lceil{#1}\right\rceil}

\newcommand{\abs}[1]{\left\lvert{#1}\right\rvert}

\newcommand\numberthis{\addtocounter{equation}{1}\tag{\theequation}}
\newcommand{\Exp}[1]{\mathbb{E}\lrp{#1}}
\newcommand{\E}[2]{\mathbb{E}_{#1}\lrp{#2}}

\newcommand{\inv}{^{\text{-}1}}

\newcommand{\mydate}{\today}

\title[Regret Minimization in Heavy-Tailed Bandits]{Regret Minimization in Heavy-Tailed Bandits}

\usepackage{times}

\coltauthor{%
 \Name{Shubhada Agrawal} \Email{shubhada.agrawal@tifr.res.in}\\
 \addr TIFR, Mumbai, India
 \AND
 \Name{Sandeep Juneja} \Email{juneja@tifr.res.in}\\
 \addr TIFR, Mumbai, India %
  \AND
 \Name{Wouter M. Koolen} \Email{wmkoolen@cwi.nl}\\
 \addr CWI, Amsterdam, Netherlands%
}

\begin{document}

\maketitle
{\bf \centering\mydate\\}
\begin{abstract}%
We revisit the classic regret-minimization problem in the stochastic multi-armed bandit setting when the arm-distributions are allowed to be heavy-tailed. Regret minimization has been well studied in simpler settings of either bounded support reward distributions or distributions that belong to a single parameter exponential family. We work under the much weaker assumption that the moments of order \((1+\epsilon)\) are uniformly bounded by a known constant \(B\), for some given \( \epsilon > 0\). We propose an optimal algorithm that matches the lower bound exactly in the first-order term. We also give a finite-time bound on its regret. We show that our index concentrates faster than the well known truncated or trimmed empirical mean estimators for the mean of heavy-tailed distributions. Computing our index can be computationally demanding. To address this, we develop a batch-based algorithm that is optimal up to a multiplicative constant depending on the batch size. We hence provide a controlled trade-off between statistical optimality and computational cost.
\end{abstract}


\section{Introduction}
In this paper, we consider the problem of sequential allocation of resources in an uncertain environment. The player is presented with \(K\) \textit{arms}, which correspond to the \(K\) unknown probability distributions. When the player selects an arm, she observes a sample generated independently from the corresponding underlying distribution, called \textit{reward}. The player's aim is to maximize the average cumulative reward, which is equivalent to minimizing the expected regret, defined to be the shortfall between the cumulative expected reward of the player
and the expected reward collected by the policy playing the arm with the maximum mean in all the rounds.

Regret minimisation in the stochastic multi-armed bandit model was first studied by \cite{Thomp1933} in the context of designing efficient clinical trials, and by \cite{robbins1952}. \cite{LaiRobbins85} proposed a lower bound on the expected regret for parametric distributions and gave a framework for optimal strategies in this setting. This lower bound was later generalized by \cite{BURNETAS1996} (see \citet[Chapter 16]{lattimore2020bandit} for a proof of the lower bound). Since then this problem has been well studied in the literature (see, e.g., \cite{Auer2002,agrawal1995sample,HondaBounded10,honda2011asymptotically,SAgrawal2011TS1,garivier2011kl,kaufmann2012TS,cappe2013kullback,bubeck2013heavytail,honda2015SemiBounded}). We refer the reader to \cite{bubeck2012Survey} for a survey on the extensive literature on this problem and its variations.


Apart from clinical trials, this regret-minimization framework finds applications in many other settings including in advertisement selection, recommendation systems, internet routing and congestion control, etc. Despite the vast body of literature, the setting where the underlying distributions are heavy-tailed (disributions for which the moment-generating function is not defined for any \( \theta  > 0\)) has been largely unaddressed. In most of the previous work, authors restrict the arm distributions to have bounded support or belong to a restrictive single parameter exponential family (SPEF) of distributions. In many bandit application domains such as in financial markets, or in congestion control over networks, it is rarely the case that the arm distributions are either parametric or bounded. Thus, it is important to understand and develop a general theory and efficient (both, computationally and statistically) algorithms that have wider applicability. 

Recently, there has been some interest beyond SPEFs. \cite{bubeck2012Survey} consider the non-parametric class of ``sub-$\psi$'' distributions, where the convex function \(\psi\) bounds the log-moment generating function of the arm-distributions. The \(\psi\)-UCB algorithm proposed by the authors is order-optimal. \cite{bubeck2013heavytail} propose algorithms which are optimal up to constants for heavy-tailed distributions. The setting considered by the authors is closest to ours. Later, we compare performance of our algorithm to the one proposed in \cite{bubeck2013heavytail}. They show that by using more robust estimators for the mean, as compared to the empirical average, one can achieve sub-linear expected regret. \cite{vakili2013deterministic} also consider bandits with heavy-tailed distributions. They propose a strategy which is based on dividing the time into interleaving sequences for exploration and exploitation. \cite{lattimore2017scale} considers distributions with a known uniform bound on the kurtosis. \cite{Cowan2018} and \cite{cowan2015asymptotically} consider Gaussian bandits with unknown mean and variance, and uniform bandits with unknown support, respectively. However, none of these algorithms exactly match the lower bound on the expected regret to the first order.

We allow for distributions that have their \( (1+\epsilon)^{th} \)-moment uniformly bounded by a constant, \(B\), for \( \epsilon > 0 \). However, existence of higher moments is not guaranteed. We develop an algorithm that suffers regret which asymptotically matches the lower bound exactly, up to the first order term. We also give a finite time analysis for the regret incurred by the algorithm. We look at the computational complexity of the algorithm and demonstrate a trade-off in the expected regret and the computational cost suffered by the proposed algorithm.

As in \cite{pmlr-v117-agrawal20a}, it can be shown that if no restrictions are imposed on the class of arm-distributions, then the lower bound on expected regret is unbounded. To make the problem learnable, we impose a mild condition on the class under consideration, and propose an index-based algorithm that is asymptotically optimal (as  the number of trials, T, tends to \(\infty\)). In particular, we focus on the class 
\[ \mathcal{L}_B \triangleq \lrset{\eta\in\mathcal{P}(\Re) : \mathbb{E}_{\eta}\abs{X}^{1+\epsilon} \leq B }, \numberthis\label{eq:LB} \]
where \(\mathcal{P}(\Re)\) denotes collection of all probability measures on \(\Re\), \(B > 0\) and \(\epsilon > 0\) are known constants, and \( \mathbb{E}_{\eta}\abs{X}^{1+\epsilon} :=\int \abs{y}^{1+\epsilon} d\eta(y) \) denotes the \((1+\epsilon)^{th}\) moment of \(\eta\).

As is common in the stochastic-bandit literature, the performance guarantees of the algorithm involve proving convergence results, which are typically a consequence of the Chernoff-Hoeffding like inequalities. However, direct application of Hoeffding's type results is not valid in our setting. We develop non-asymptotic concentration inequalities for functionals of probability measures that appear in the lower bound, which may also be of independent interest. Our approach for constructing an index for each arm can be used to get tight confidence intervals for means of heavy-tailed distributions. In particular, we show that the confidence intervals based on the widely used ``truncated'' or ``trimmed'' empirical mean contain our intervals, showing the superiority of our approach. See \cite{lugosi2019mean} for popular estimators for a distribution's mean in the heavy-tailed setting. We also demonstrate numerically that our algorithm suffers significantly less regret compared to the Robust-UCB algorithm of \cite{bubeck2013heavytail}, which derives its index from the above mentioned estimators for mean.

Computing the index can be very costly in this generality. To address this, we propose a batched version of the algorithm that computes the index for each arm only at the beginning of each batch, and allocates all the samples within that batch to the arm with maximum value of the computed index. We show that with carefully chosen batch sizes, this batched-algorithm suffers regret that is off by only a constant multiplicative factor, while significantly improving the computational cost. 

Since we allow for heavy-tailed distributions, we can no longer identify the distributions with one parameter, as is the case in the most widely used setting of SPEF. We work in the space of probability measures, where we use L\'evy metric (or equivalently the topology of weak-convergence) to define the notion of convergence.

We also establish the conjectured optimality of Empirical KL-UCB algorithm of \cite{cappe2013kullback} and give the first optimal finite-time regret bounds for bounded-support arm distributions. In \cite{garivier2011kl}, the authors gave a finite-time bound for a modification of the algorithm, and established its optimality only in the special case of Bernoulli arms. \cite{maillard2011finite} independently gave a tight finite time analysis for the Bernoulli case. For general bounded-support distributions, \cite{HondaBounded10} proposed an asymptotically optimal algorithm, but did not give finite-time bounds on the regret. 

In a nutshell, we propose the first asymptotically optimal algorithm for the heavy-tailed setting that matches the lower bound exactly up to the first order term. In this generality, the computational cost incurred by the algorithm can be significant. We propose a modification of the optimal algorithm that matches the lower bound up to constants, but requires significantly less computational effort. Our index gives tighter confidence intervals for the mean of heavy-tailed distributions compared to the well-known truncation or trimming based estimators. Moreover, the finite time analysis presented can be used to establish finite time guarantees for some of the existing algorithms.

\textbf{Roadmap:} In Section \ref{Sec:BackgroundLB} we describe the setup and discuss lower bound for the regret-minimization problem. Our proposed algorithm is presented in Section \ref{Sec:alg}. Section \ref{Sec:results} contains our main results for the proposed algorithm, including the finite time guarantee and its asymptotic optimality. A modification of the original algorithm that is practically tuned but at the cost of optimality up to constants, is also presented in this section. Regret analysis of the algorithm and proof ideas for its theoretical guarantee are presented in Section \ref{Sec:regret_analysis}. Superiority of our algorithm over that of \cite{bubeck2013heavytail} is established in Section \ref{Sec:ComparisonWithTruncatedMean}. In this section we also show that our confidence intervals are tighter compared to popular mean estimators for heavy tailed distributions. We discuss the trade-off in the computational cost and statistical optimality of the algorithm in Section \ref{Sec:ComputationalCost}. Towards the end, in Section \ref{Sec:Numerics}, we present the results of the numerical experiments. We conclude in Section \ref{Sec:conclusions}. 

\section{Background and the lower bound}\label{Sec:BackgroundLB}
Let \(\mathcal{ P}(\Re)\) denote the collection of all probability measures on \(\Re\). For \(\eta\in\mathcal{P}(\Re)\), let \(m(\eta) = \int\nolimits_{\Re} x d\eta\) denote the mean of distribution \( \eta\). Furthermore, let \(\KL\lrp{\eta,\kappa} =  \int\log\lrp{\frac{d\eta}{d\kappa}}d\eta(x)  \) denote the Kullbeck-Leibler divergence from probability distribution \( \eta \) to \(\kappa \). Let  \(\mathcal{L} \) be any collection of probability measures, and  let \(\mathcal L^K\) denote the collection of vectors of \(K\) distributions, each belonging to \(\mathcal{L}\). 

Given \( \mu \in \mathcal L^K\) such that  \(\mu =  \lrset{\mu_1,\mu_2,\dots,\mu_K}\), we assume for convenience that arm \(1\) has the maximum mean. Let the arm selected by the algorithm at time \(n\) be denoted by \(A_n\) and let the observed random sample at that time be denoted by \(X_{n}\). Then \(X_n\) is an independent sample distributed according to \(\mu_{A_n}\). Define the expected regret till time \(T\) as follows:
\begin{equation}\label{eq:exp.regret}
  \Exp{R_T} \triangleq \sum\limits_{n=1}^{T} \lrp{m(\mu_1)-m(\mu_{A_n})} = \sum\limits_{a=1}^K \Exp{N_a(T)} \Delta_a,
\end{equation}
where \(N_a(T)\) denotes the number of times arm \(a\) has been pulled in \(T\) trials, and \(\Delta_a := m(\mu_1)-m(\mu_a)\) is the sub-optimality gap of arm \(a\). For distribution \(\eta \in \mathcal{P}(\Re)\) and candidate mean \(x\in\Re\), we define
\[\KLinfR{\mathcal{L}}(\eta, x) := \inf\lrset{\KL(\eta,\kappa):\kappa\in\mathcal{L} ~\text{ and }~m(\kappa) \geq x }. \numberthis \label{eq:GeneralKLinf} \]
Then, \cite{BURNETAS1996} show that any reasonable strategy acting on a bandit problem \(\mu \in \mathcal L^K \), for any collection \(\cal L\),  suffers expected regret satisfying:

\begin{equation}\label{eq:RegLB}
\liminf\limits_{T\rightarrow \infty } \frac{\Exp{R_T}}{\log(T)} \geq \sum\limits_{a\ne 1}\frac{\Delta_a}{\KLinfR{\mathcal{L}}(\mu_a, m(\mu_1))},
\end{equation} 
where \(R_T\) is the random variable which denotes the total regret suffered by the algorithm till time \(T\). The proof of lower bound (\ref{eq:RegLB}) relies on change of measure arguments (see \cite{lattimore2020bandit}).  \citet[Lemma 1]{pmlr-v117-agrawal20a} show that it is necessary to impose certain restrictions on the class \(\mathcal{L} \) under consideration, otherwise \(\KLinfR{\mathcal{L}}(\cdot, \cdot) = 0 \) leading to unbounded expected regret. To this end, we restrict \(\mathcal{L} \) to class \(\mathcal{L}_{B}\) defined in \eqref{eq:LB}, i.e., to the collection of all distributions satisfying \(\mathbb{E}({\abs{X}^{1+\epsilon}})\leq B\), for fixed positive constants \(\epsilon\) and \(B\). Notice that  in order to bound the expected regret in \eqref{eq:exp.regret}, it is sufficient to bound the expected number of pulls of each sub-optimal arm $a \neq 1$.

Building upon the algorithms with \(\KL\)-based confidence intervals in \cite{maillard2011finite}, \cite{garivier2011kl}, and \cite{cappe2013kullback}, where the authors propose optimal algorithms for much simpler settings of either Bernoulli or SPEF arms, we develop an algorithm that matches the lower bound in (\ref{eq:RegLB}), in a much more general setting, allowing for heavy-tailed distributions.

As mentioned in the introduction, we study the convergence of sequences of probability measures in the L\'evy metric, which is reviewed in Appendix \ref{App:KLinfProp}. The analysis of the algorithm uses the continuity and convexity properties of \(\KLinf\), which are also proven in the Appendix \ref{App:KLinfProp} (Lemma \ref{lem:PropKLinf}). 

\section{The algorithm}\label{Sec:alg}
Our algorithm follows the UCB template with two variations. First, our arm indices are constructed from deviation inequalities for $\KLinfR{\mathcal{ L}}$, which allows us to show that our algorithm matches the (asymptotic) instance-optimal regret lower bound for the heavy-tailed setting. Second, our algorithm processes the samples in batches. Our geometric batching allows us to reduce the worst-case run-time from essentially $O(n^2)$ for the sample-at-a-time case to $O(n\log\log(n))$, at the cost of a mild factor in our regret.

Algorithm \ref{algorithm:gen} formally describes our algorithm, \( \KLinf \)-UCB. This is an index-based algorithm that proceeds in batches. After each batch, it computes an index for each arm and allocates the next batch of samples to the arm with the maximum index (breaking ties arbitrarily, if any). Henceforth, unless specified, we set \(\mathcal L = \mathcal L_B\), as defined in (\ref{eq:LB}). Furthermore, when \(\cal L\) is clear from the context, for ease of notation, we denote the functional \(\KLinfR{\mathcal{ L}}\), defined in (\ref{eq:GeneralKLinf}), by \(\KLinf\).

Let \(\mu\in\mathcal L^K\) be the given bandit instance, \(U_a(n)\) denote the value of the index corresponding to arm \(a\) at time \(n\), \(A_n \) denote the arm selected at time \(n\), and let \([K]\) denote the set \(\lrset{1,\dots, K}\). Let \(\hat{\mu}_a(n)\) denote the empirical distribution corresponding to \(N_a(n)\) samples from arm \(a\), and \(\tilde{\eta} \ge 0\) be a multiplicative factor that will be used to determine the batch sizes. The algorithm takes as inputs \(K\), \( \tilde{\eta} \), \(B\),  \( \epsilon \), and a threshold function, \(g_a(.)\) corresponding to each arm, which will be used for computing the index for that arm.

\begin{algorithm}
	\DontPrintSemicolon
	\SetKwInOut{Input}{Input}\SetKwInOut{Output}{Output}
	\Input{\(K;~\) description of \( \cal L \), i.e., \(B\) and \(\epsilon;~\) \(\tilde{\eta};~ \) threshold functions for each arm, i.e., \(g_a(\cdot)\).}
	\BlankLine
	\textbf{Initialization:} Allocate \(1\) sample to each of the \(K\) arms. \\
	Set \(n \longleftarrow K+1 \), \(~ j \leftarrow K+ 1\).\\ 
	Store empirical distributions, \(\hat{\mu}_a(n)\), and update \(N_a(n)\) for all arms \(a\in [K]\).  \\
	\While{True}{
          Compute index \( U_a(n) = \sup\lrset{x\in\Re: N_a(n)\KLinf(\hat{\mu}_a(n),x )\leq g_a(n)} \) for each arm. \;

          Compute best arm \( A_n = \argmax_{a\in [K]} U_a(n) \) and batch size \(~B_j = \max\lrset{1, {\ceil{\tilde{\eta} N_{A_n}(n)}}}\).\;

          Sample arm \(A_n\) for \(B_j\) many trials and set \( n\leftarrow n + B_j \), and \(~j \leftarrow j + 1\).\;

          Update \( \hat{\mu}_{a}(n) \) and \(N_a(n)\) for each arm.
	}
	\caption{\(\KLinf\text{-UCB}(K,B,\epsilon,\tilde{\eta}, \lrset{g_a(\cdot)}_{a=1}^K)\).}\label{algorithm:gen}
\end{algorithm}

Our index $U_a(n)$ for arm \(a\) at time \(n\) is based on the functional \(\KLinf\). It approximately corresponds to the inverse of \( N_a(n)\KLinf(\hat{\mu}_a(n), .) \), evaluated at the threshold, \(g_a(n)\) and can be re-expressed as
\[ U_a(n) = \max\lrset{\E{\eta}{X}: ~ \eta\in\mathcal L, ~ N_a(n)\KL(\hat{\mu}_a(n),\eta) \le g_a(n)}. \numberthis \label{eq:index_reformulated} \]
It is the maximum mean among distributions in \( \cal L \) that are close to the empirical distribution in \( \KL \) divergence. This formulation of the index will be useful in comparing our confidence widths to those of the truncated empirical mean estimator (see Section \ref{Sec:ComparisonWithTruncatedMean}).  Before looking at the computational cost incurred by the algorithm, we look at its theoretical guarantees.

\subsection{Main results}\label{Sec:results}
Let \(B_j \) be the random variable denoting the size of the \(j^{th}\) batch. Theorem \ref{th:geometric} below gives a finite-time bound on the number of pulls of a sub-optimal arm by the proposed-algorithm for appropriately chosen threshold functions \(g_a(\cdot)\). Corollary \ref{cor:asympOpt} shows that the \(\KLinf\)-UCB algorithm is asymptotically optimal up to a multiplicative factor of \(\lrp{1+\tilde{\eta}}\), and matches the lower bound when the batch size is \(1\), i.e., \(\tilde{\eta} = 0\). However, from our discussion in Section \ref{Sec:ComputationalCost}, for \(\tilde{\eta} = 0\), the algorithm has a quadratic computational cost. On the other hand, for \(\tilde{\eta}>0\), the computational cost reduces to being almost linear in number of samples, at the cost of matching the lower bound upto a constant factor of \( 1+\tilde{\eta} \). Thus, there is a trade off between the cost of computation and the expected regret suffered by the algorithm in long run. In Section \ref{Sec:Numerics}, numerical experiments demonstrate that even with the sub-optimality factor of \( (1+\tilde{\eta}) \), our algorithm suffers significantly less regret compared to the Robust-UCB algorithm with truncated empirical mean as estimator for true mean, proposed in \cite{bubeck2013heavytail}.
Our main regret bound is the following:

\begin{theorem}\label{th:geometric} 
	Let \(T > K \ge 2\), \(\mu\in\mathcal L^K \), \(\tilde{\eta} \geq 0\), and \(g_a(t) = \log(t)+2\log\log(t)+2\log(1+N_a(t))+1\). For each sub-optimal arm \(a\), \(\KLinf\)-UCB\( \lrp{K,B,\epsilon,\tilde{\eta},g_a\lrp{\cdot}} \)  has \(\Exp{N_a(T)} \) at most 
	\begin{equation*}
	(1+\tilde{\eta}) \lrp{\frac{\log T}{\KLinf(\mu_a,m(\mu_1))} + \frac{3(\log T)^{2/3} (c'_\mu)^{1/3}}{2(\KLinf(\mu_a,m(\mu_1)))^{4/3}} + O((\log  T)^{1/3}) + O(\log\log(T))} .
      \end{equation*}
\end{theorem}
In Theorem \ref{th:geometric} above,  \(c'_\mu > 0\) is a bandit instance-dependent constant. Exact \(O((\log(T))^{1/3}) \) and \(O(\log\log T)\) terms are given in (\ref{eq:exactRegret}) below. Theorem \ref{th:geometric} implies logarithmic regret for \(\KLinf \)-UCB, and Corollary~\ref{cor:asympOpt} shows its asymptotic instance-optimality. 

\begin{corollary}\label{cor:asympOpt}
	For \(\mu\in \mathcal L^K \), \( \tilde{\eta} \ge 0 \), and \( g_a(t) = \log(t) + 2\log\log(t) + 2\log(1+N_a(t)) + 1 \), \(\KLinf \)-UCB, with inputs $(K,B,\epsilon,\tilde{\eta},g_a(.))$ is asymptotically optimal up to a factor of \( (1+\tilde{\eta}) \), i.e., 
 	\begin{equation*}
	 \limsup\limits_{T\longrightarrow \infty}\frac{\Exp{N_a(T)}}{\log(T)} \leq \frac{1+\tilde{\eta}}{\KLinf(\mu_a,m(\mu_1))}. 
	 \end{equation*}
\end{corollary}

Numerically we observe that the threhold used by \( \KLinf \)-UCB is conservative (see Figure \ref{fig:KLinfUCBEasy} in Section \ref{Sec:Numerics}). A version with lower thresholds performs significantly better in the time horizon considered. To address this, we propose a closely related algorithm, \( \KLinf \)-UCB2, which has much smaller threshold, and suffers regret with the first order term that is arbitrarily close to that in the lower bound. 

Let \( \epsilon_1 > 0 \), let \(\tilde{B}=B+\epsilon_1\) and define \( \delta_t = \log(1+(\log\log(t))\inv) \). For \( \mu\in\mathcal L_B \), \( \KLinf \)-UCB2 is precisely \(\KLinf  \)-UCB$(K,\tilde{B},\epsilon,\tilde{\eta},(1+\delta_t)^2\log(t))$.  For \(\eta\in\mathcal P(\Re)\), and \(x\in\Re\), let \( \KLinfR{\epsilon_1}(\eta,x) \) denote \(\KLinfR{\mathcal L_{\tilde{B}}}(\eta,x)\). 
\begin{theorem}\label{th:AsymptoticOptimality2}
	For \( \mu\in\mathcal L^K, \) \( \tilde{\eta} \ge 0 \), \(\epsilon_1>0\), and \( g_a(t) = (1+\delta_t)^2\log(t) \), \( \KLinf \)-UCB2 satisfies
	\[ \limsup\limits_{T\rightarrow \infty} \frac{\Exp{N_a(T)}}{\log(T)} \le \frac{1+\tilde{\eta}}{\KLinfR{\epsilon_1}(\mu_a,m(\mu_1))}. \]
\end{theorem}
In Lemma \ref{lem:PropKLinf}, Appendix \ref{App:KLinfProp}, we show that \( \KLinfR{\mathcal L_B} \) is a continuous function of the class parameter, \(B\). Hence, by choosing \( \epsilon_1 \) close to \(0\),  \( \KLinf \)-UCB2 regret gets arbitrarily close to the lower bound, as \(T\rightarrow \infty\) (modulo \( (1+\tilde{\eta}) \) factor). 

\subsection{Conjecture and open problem of \cite{cappe2013kullback}: }
The analysis of the proposed algorithm can be specialized to bound the regret of the Empirical KL-UCB algorithm of \cite{cappe2013kullback} for arm-distributions with bounded support. In this setting, \(\mathcal L = \mathcal{P}([0,1]) \), the collection of all probability measures supported on \([0,1]\). For \( \mu_a\in \mathcal L\), for each arm \(a\in[K]\), and \( x\in [0,1] \), on setting
\[ \KLinfR{\mathcal L}(\mu_a, x) =  \inf~~\KL(\mu_a,\kappa) \quad \text{s.t.}\quad \kappa \in \mathcal{P}([0,1]), ~~ m(\kappa) \ge x,\numberthis\label{eq:Klinf_bddSupp} \]
in the index of our algorithm, we recover the Empirical KL-UCB algorithm.
\begin{proposition}\label{prop:BddSupp}
  Let \(\mathcal L = \mathcal P([0,1])\). Empirical KL-UCB, with \( g_a(t) := \log t + \log \log t \) bounds the pull counts for suboptimal arms $a>1$ at \( T \ge K \) by
  \[
    \Exp{N_a(T)}
    ~\le~
    \frac{\log(T)}{\KLinfR{\mathcal L}(\mu_a,m(\mu_1)) - O\lrp{\lrp{\log T}^{-\frac{1}{5}} \lrp{\log\log T}^{\frac{1}{5}}}}
    + O\lrp{\lrp{\log T}^{\frac{4}{5}}\lrp{\log\log(T)}^\frac{1}{5}}. \]
\end{proposition}

 \cite{HondaBounded10} develop an explicit representation and properties for the functional defined in (\ref{eq:Klinf_bddSupp}) above, which we review in Appendix \ref{app:sec:finiteBoundForBdd}. Carefully using these in our analysis, we get the above mentioned bound. We refer the reader to Appendix \ref{app:sec:finiteBoundForBdd} for the exact bound and its proof.

\subsection{Regret analysis}\label{Sec:regret_analysis}
In this section, we prove Theorem \ref{th:geometric}. Proof of Theorem \ref{th:AsymptoticOptimality2} is similar, and is given in Appendix \ref{app:perturbations}. 

The proof proceeds by analysing the events leading to the selection of a sub-optimal arm $a >1$ by Algorithm \ref{algorithm:gen}. We show that if it has been sampled enough, then the probability of it getting selected is extremely small. In particular, this corresponds to showing 2 things: first, \(\KLinf \)-UCB index is a high probability upper bound on the true-mean, and second, probability of it being too large is small. 

Notice that for \(x\in\Re\), and \(b\in [K]\), if \(N_b(n)\KLinf(\hat{\mu}_b(n), x)\) is at least the threshold, \( g_b(n)\), then the index for arm \(b\) at time \(n\) is smaller than \(x\). Similarly, if \(N_b(n)\KLinf(\hat{\mu}_b(n),x)\) is less than the threshold, then the index computed by \(\KLinf\)-UCB is at least \(x\). Proposition \ref{prop:deviation_self_normalizing} below shows that for all \(n\), our index is an upper bound on the true mean of the arm-distribution, with high probability.

\begin{proposition}\label{prop:deviation_self_normalizing}
	For $x \ge 0$, \( b\in[K] \), 	$$\mathbb{P}\lrp{\exists n\in\mathbb{N}: N_b(n)\KLinf(\hat{\mu}_b(n), m(\mu_b) ) - 1 - 2\log(1+N_b(n))  \ge x } \le e^{-x}. $$
\end{proposition}
To prove the above proposition, we construct mixtures of super-martingales that dominate the exponential of \( N_b(n)\KLinf(\hat{\mu}_b(n),m(\mu_b)) \) adjusted by \( 1+2\log(1+N_b(n)) \). The dual formulation of \( \KLinf \) (see Appendix \ref{App:KLinfProp}) suggests a natural candidate for these super-martingales. We refer the reader to Section \ref{sec:proof:prop:deviation_self_normalizing} in the appendix for details of the proof of Proposition \ref{prop:deviation_self_normalizing}.

We next show that \(\KLinf\)-UCB index being too large is a rare event. Let \( \hat{\mu}_{a,s} \) denote the empirical distribution corresponding to \(s\) samples from arm \(a\). Lemma \ref{lem:BoundOfE2j} below, will be used to bound the probability that the index for sub-optimal arm $a$ takes value close to the mean of the best-arm, after sufficient samples have been allocated to it.

\begin{lemma}\label{lem:BoundOfE2j}
	For \(\mu\in\mathcal L^K\), \(\delta > 0\), \(\delta < \min_{a\ne 1} \KLinf(\mu_a,m(\mu_1))\), and \(b\in [K] \), there exists \( c_\mu > 0 \)  such that 
	\[ \mathbb{P}\lrp{ \KLinf\lrp{\hat{\mu}_{b,s}, m(\mu_1) } \leq \KLinf\lrp{\mu_b,m(\mu_1)} - \delta} \leq e^{-s~c_\mu\delta^2}.\]
\end{lemma}
Proof of Lemma \ref{lem:BoundOfE2j} above relies on the dual formulation of \( \KLinf \) and its properties, which we review in Section \ref{App:KLinfProp}. See Appendix \ref{proof:lem:BoundOfE2j} for proof of Lemma \ref{lem:BoundOfE2j}.

\subsubsection{Proof of Theorem \ref{th:geometric}}
We now outline the proof of Theorem \ref{th:geometric} for the case when \(\tilde{\eta} > 0\). When \( \tilde{\eta} = 0\), the proof follows similarly  with suitable adjustments. Fix \(T \geq K+1\). Let \(T_j =K+1+ \sum\nolimits_{i=K+1}^{j-1}B_i \) denote the random time marking the beginning of the \(j^{th}\) batch, where we recall that \(B_i\) is the random variable denoting the number of samples to be allocated in the \(i^{th}\) batch. In this section, for simplicity of notation, we denote \(m(\mu_1)\) by \(m\). The event that at the beginning of \( j^{th} \) batch, sub-optimal arm, \(a\), has the maximum index, i.e., \(\lrset{A_{T_j}=a}\) for \(a\ne 1\), equals 
\begin{align*}
\lrset{U_1(T_j) \leq m ~~\text{ and }~~ A_{T_j}=a }~ \bigcup ~\lrset{U_a(T_j) > m ~~\text{ and }~~ A_{T_j}=a }, \numberthis\label{eq:IncorrectSelectionEvent}
\end{align*}

where the first event corresponds to the index for arm \(1\) evaluating smaller than its true mean at a stopping time \(T_j\), while the other one corresponds to the index for the sub-optimal arm taking values higher than the mean of the optimal arm.

Let \(N\) be the random number of batches allocated by the algorithm till time \(T\). Recall that initial \(K\) batches correspond to each arm being pulled once. Then, \( N_a(T) \) equals \(1+ \sum_{j=1}^N B_j \mathbbm{1}\lrp{A_{T_j}=a} \), and 
\begin{equation}\label{eq:NumberOfSubOptPulls}
\Exp{N_a(T)} =  ~ 1 + \Exp{D_N} + \Exp{E_N},
\end{equation}
where, using the division from (\ref{eq:IncorrectSelectionEvent}), we define
\[ D_N := \sum\limits_{j=K+1}^{N}B_j \mathbbm{1}\lrp{U_1(T_j) \leq m, ~ A_{T_j} = a}, ~\text{and }~ E_N := \sum\limits_{j=K+1}^{N}B_j \mathbbm{1}\lrp{U_a(T_j)> m, ~ A_{T_j}=a}. \]

Let us now look at the deviation of arm \(a\), which will contribute to the dominant term in regret of \( \KLinf \)-UCB. 

{\noindent \bf Controlling the deviations of sub-optimal arm-\( \Exp{E_N} \):}
From the definition of index of the algorithm, for \(t\ge K+1\) and \(x\in\Re\), the event \(\lrset{U_a(t) \ge x}\) is same as \(\lrset{N_a(t) \KLinf(\hat{\mu}_a(t),x) \le g_a(t) } \). Fix \(\delta>0\) satisfying \( \min_{a > 1} \KLinf\lrp{\mu_a, m} \geq  \delta\). Clearly, \( \mathbbm{1}\lrp{U_a(T_j) \ge m,~A_{T_j} = a} \) is dominated by the sum of \(E_{1j}\) and \(E_{2j}\) defined below:
\[E_{1j} = \mathbbm{1}\lrp{\KLinf\lrp{\hat{\mu}_a(T_j), m} \leq \frac{g_a(T_j)}{N_a(T_j)},~\KLinf\lrp{\hat{\mu}_a(T_j), m} > \KLinf(\mu_a,m)-\delta,~ A_{T_j} = a },\]
\[E_{2j} = \mathbbm{1}\lrp{\KLinf\lrp{\hat{\mu}_a(T_j), {m}} \leq \KLinf(\mu_a,{m})-\delta, ~ A_{T_j}=a }.\]
Whence, 
\begin{align*}
E_N \le &\sum\limits_{j=1}^{N}B_j E_{1j}+\sum\limits_{j=1}^{N}B_jE_{2j}.\numberthis \label{eq:Ct}
\end{align*}

We argue that \(E_{1j}\), summed over all the batches till time \(T\), contributes the first order term in the regret. Clearly, it is dominated by $\mathbbm{1}\lrp{N_a(T_j)\lrp{\KLinf\lrp{\mu_a,m}-\delta} \leq g_a(T_j), A_{T_j}=a  },$  giving 
\[\sum\limits_{j=1}^{N}B_jE_{1j} \leq \sum\limits_{j=1}^{N}B_j\mathbbm{1}\lrp{{N_a(T_j)} \leq \frac{g_a(T_j)}{\KLinf\lrp{\mu_a,m}-\delta}, A_{T_j}=a  }.\]
Lemma \ref{lem:BoundingSumE1j} in Appendix \ref{proof:lem:BoundingSumE1j} essentially bounds the r.h.s. above, giving 
\[		\sum\limits_{j=1}^{N}B_jE_{1j}\le (1+\tilde{\eta})\lrp{\frac{\log(T)}{\KLinf(\mu_a,m)-\delta}+ O\lrp{\log\log(T)}}.
\]
The exact form of the \( O(\log\log(T)) \) term above is given in proof of Lemma \ref{lem:BoundingSumE1j}. 

Next, for \(k\ge 2\), let \( T^{k}_{a} \) denote the random time of beginning of the batch when arm a won for the \(k^{th}\) time. In particular, arm \(a\) has been sampled for \((k-1)\) batches till this time. 
Thus, \( T^k_a\) is at least \(K-1 + 1+\tilde{\eta}+\dots+\tilde{\eta}(1+\tilde{\eta})^{k-3} =K-1+(1+\tilde{\eta})^{k-2} \). Moreover, let \(N_{B,a}\) denote the total number of batches allocated to arm \(a\) till time \(T\). The other term in (\ref{eq:Ct}) satisfies 
\[ \Exp{\sum\limits_{j=1}^N B_j E_{2j}} = \Exp{\sum\limits_{k=2}^{N_{B,a}}  B_{T^k_a} \mathbbm{1}\lrp{\KLinf( \hat{\mu}_a(T^k_a), m) \le \KLinf(\mu_a,m) -\delta}}. \]

Clearly, \( N_a(T^k_a) \) is deterministic. For \(k\ge 2\) and \(\tilde{\eta}> 0\), it is at least \((1+\tilde{\eta})^{k-1}\), and at most \(((1+\tilde{\eta})^{k-1})\tilde{\eta}\inv\). Lemma \ref{lem:BoundOfE2j} bounds the expectation of the indicator random variable in the above expression. Lemma \ref{lem:SumE2jIntermediate} in appendix shows that the bound in this case is proportional to \((1+\tilde{\eta})/(c_\mu \delta^2 )  \), where \(c_\mu\) is the bandit instance-dependent constant from Lemma \ref{lem:BoundOfE2j}.  
Thus,  
\begin{align*}
\Exp{E_N} &\leq  (1+\tilde{\eta})\lrp{\frac{\log(T)}{\KLinf\lrp{\mu_a,m}-\delta} + \frac{o(1)}{c_\mu\delta^2}+O(\log\log(T))},\label{eq:CT} \numberthis
\end{align*}
where the \(O(\log\log T)\) terms in the above expression correspond to those in Lemma~\ref{lem:BoundingSumE1j}, and \( o(1) \) is specified in Lemma \ref{lem:SumE2jIntermediate}. 

{\noindent \bf Controlling the downward deviation of the optimal arm-\(\Exp{D_N}\): }
This term only contributes a constant to the regret till time \(T\). We refer the reader to Lemma \ref{lem:BoundingDN} in Appendix \ref{app:OptimalArmDeviations} for a proof.

Combining everything, we get 
\[ \Exp{N_a(T)} \le  (1+\tilde{\eta})\lrp{\frac{\log(T)}{\KLinf\lrp{\mu_a,m}-\delta} + \frac{o(1)}{c_\mu\delta^2}+O(\log\log(T))}, \numberthis\label{eq:sunopt_finite_time} \]
where the \(O(\log\log T)\) terms in the above expression correspond to those in Lemma~\ref{lem:BoundingSumE1j} and constant in Lemma \ref{lem:BoundingDN}, and \( o(1) \) is specified in Lemma \ref{lem:SumE2jIntermediate}. The above bound can be optimized over \( \delta \). Setting \( \delta \) to \( (c'_\mu (\KLinf(\mu_a,m(\mu_1)))^2 / \log T)^{1/3}, \) where \(c'_\mu = 2o(1)/c_\mu \) we get that \( \Exp{N_a(T)} \) is at most 
\[(1+\tilde{\eta}) \lrp{\frac{\log T}{\KLinf(\mu_a,m)} + \frac{3(\log T)^{2/3} (c'_\mu)^{1/3}}{2(\KLinf(\mu_a,m))^{4/3}} + O((\log  T)^{1/3}) + O(\log\log(T))}, \numberthis\label{eq:exactRegret} \]
where the \(O(\log\log T)\) terms in the above expression correspond to those in Lemma~\ref{lem:BoundingSumE1j}, and \( O((\log T)^{1/3}) \) corresponds to \( ((\log T)^{1/3} \delta^2 / (\KLinf(\mu_a,m))^3) \sum_i (\delta/\KLinf(\mu_a,m))^i \). 

\subsection{Comparison with Robust-UCB with truncated empirical-mean \citep{bubeck2013heavytail}}\label{Sec:ComparisonWithTruncatedMean}
It is well known that the standard empirical mean does not concentrate fast when the underlying distributions are heavy-tailed. \citet{bubeck2013heavytail} review three different estimators that concentrate exponentially fast, and propose a UCB algorithm based on these. In this section, we compare our algorithm with that based on the truncated empirical mean estimator (referred to as Robust-UCB, in this section), at the level of confidence intervals around the true mean. The other two estimators proposed by the authors are under different assumptions on the arm-distributions, and are not directly comparable.

Fix \(\delta > 0\). Let \(\eta\in\mathcal L\), and let \(\hat{\eta}_n\) denote the empirical distribution based on \(n\) samples from \(\eta\). Recall from (\ref{eq:index_reformulated}) that our index is
\(U_{\eta}(n) = \max_{\kappa \in \cal L}~ \lrset{\E{\kappa}{X} : ~ n\KL(\hat{\eta}_n,\kappa) \le C}, \)
where \(C\) is an appropriately chosen threshold to ensure that  \(U_\eta(n) \) is an upper bound on \(m(\eta)\) with probability at least \(1-\delta\). 

The Donsker-Varadhan variational representation for \( \KL \)-divergence expresses the \(\KL \)-divergence between any two probability measures \(P,Q\), defined on a common space \(\Omega \), as 
$ \KL(P,Q) = \sup_{g}\lrset{\E{P}{g(X)} - \log \E{Q}{e^{g(X)}}  },    $
where the supremum is taken over all the measurable functions, \( g: \Omega \rightarrow \Re \), such that \( \E{Q}{e^{g(X)}} \) is well-defined. Using this to bound \(\KL(\hat{\eta}_n,\kappa) \) in our index, with \( g(X) := -\theta X\mathbbm{1}\lrp{\abs{X}\le u_n} \), where \( u_n=(Bn/\log(\delta\inv))^{1/(1+\epsilon)} \), and \( \theta > 0 \), we get that our index $U_\eta(n)$  is at most
\[ \max\lrset{\E{\kappa}{X}:\kappa\in\mathcal L \text{ and } -\theta\sum\limits_{i=1}^n X_i\mathbbm{1}\lrp{\abs{X} \le u_n } - n\log \E{\kappa}{e^{-\theta X\mathbbm{1}(\abs{X}\le u_n)}} \le C  } .\numberthis\label{eq:indexUB1} \] 
Since \( \abs{X\mathbbm{1}\lrp{\abs{X}\le u_n}} \le u_n \),  and \( \kappa\in \mathcal L\), we have \( \E{\kappa}{X^2 \mathbbm{1}\lrp{\abs{X}\le u_n}} \le  Bu^{1-\epsilon}_n \). Let
\[\hat{m}_{1T} := \frac{1}{n}  \sum_i X_i\mathbbm{1}\lrp{\abs{X_i} \le u_n}. \numberthis\label{eq:truncatedEmpMeanLastLevel} \] 
Using the standard analysis of Bernstein's inequality to optimize over \( \theta \) in (\ref{eq:indexUB1}), and substituting for \( u_n=(Bn/\log(\delta\inv))^{1/(1+\epsilon)} \), we get the following upper bound on our index:
\[ \hat{m}_{1T} +  {B}^{\frac{1}{1+\epsilon}} \lrp{\frac{\log\delta\inv}{n}}^{\frac{\epsilon}{1+\epsilon}}\lrp{1+ \lrp{e^{\frac{C}{\log\delta\inv}}-1}\frac{\log\delta\inv}{C}}. \numberthis\label{eq:truncationBound} \]

We refer the reader to Appendix \ref{App:ComparisonWithBubeck} for a proof of the above statement. When \( C=\log\delta\inv \), the above bound on our index is at most the index of the Robust-UCB algorithm (where we note that (\ref{eq:truncationBound}) has improved constants).

We can compare more precisely by considering the application of these indices within the UCB template. In the Robust-UCB algorithm, \( \delta \) is set to \(t^{-2} \), while we have \(\delta = t\inv\). This is due to our making use of anytime concentration in Proposition \ref{prop:deviation_self_normalizing}, which essentially avoids one union bound in the analysis. We set \( C = \log(t)+2\log N_a(t) + 2\log\log(t) \) in our algorithm.
For sub-optimal arms, since \(N_a(t) = O(\log(t))\), \(C = \log(t)+d\log\log(t) \), for some constant \(d\). In this case, \( C/\log(t) = O(1)\), and the bound in (\ref{eq:truncationBound}) recovers the index of Robust-UCB, which is larger than ours. On the other hand, for the optimal-arm, \( C \approx 3\log t \), since the number of pulls of the optimal  arm is linear. In this case, the Robust-UCB index may be smaller than ours. But this is harmless, as a larger index for arm \(1\) only increases its chances of being selected.

Taking a step back, we see that concentration inequalities are typically proved starting from an application of Chernoff (including \ Bernstein/Hoeffding) to get a high probability bound on the moment-generating function of some variable of interest (which may include clipping, truncating or soft versions thereof).  From there the consequences for the mean are worked out. In our approach we start instead from the event that $\kappa$ is in a KL ball around the empirical distribution $\hat \eta$. By means of Donsker-Varadhan, this provides MGF control over \emph{all} such variables simultaneously, at the price of the slightly elevated threshold $C$ on the MGF instead of the special-purpose $\ln \delta\inv$. Modulo this difference, we hence find that $\KLinf$-based indices dominate \emph{all} MGF-based indices simultaneously.

We finally remark that the truncated empirical mean estimator considered above is a minor variation of the one proposed by \citet{bubeck2013heavytail}, who truncate each sample $X_i$ at a different truncation level $u_i$, call this \( \hat{m}_{2T} \). However, repeating the analysis of \citet[Lemma 1]{bubeck2013heavytail} using \( \hat{m}_{1T} \) defined above, we in fact obtain tighter confidence intervals, and hence a tighter regret upper-bound (see, \citet[Proposition 1]{bubeck2013heavytail}). There also does not seem to be a computational advantage to using $\hat{m}_{2T}$ in the context of UCB, as the threshold $u_i$ used for sample $X_i$ depends on the employed confidence level $\delta$, which in turn depends on the current time $t$, precluding the option of maintaining $\hat{m}_{2T}$ incrementally. Numerically, we observe that Robust-UCB with these two different estimators suffers similar regret. Later, in Section~\ref{Sec:Numerics}, we numerically establish the superiority of our algorithm compared to Robust-UCB.

\subsection{Computational cost of \( \KLinf \)-UCB} \label{Sec:ComputationalCost}
In this section, we discuss the trade-off in the computational cost and statistical optimality of the proposed algorithm. We observe numerically that the time for computing \(\KLinf\lrp{\hat{\mu}_a(n),\cdot}\), increases with \(N_a(n)\). This can be seen from its dual formulation (\ref{lem:DualFormulation}), which grows by one term everytime arm \(a\) is sampled. Hence, computing the index for each arm at time \(n\), where \( \kappa \) corresponds to the empirical distribution, has a cost that is linear in \(n\). Let this linear cost be \(c_1+c_2n\), where \(c_1\in\Re\)  and \(c_2 > 0\) are constants.

If the batch size is a constant, say \(1\) (i.e., when \(\tilde{\eta}=0\)), then at each time step the algorithm evaluates the index for each arm, and plays the one with maximum value of the computed index. The total cost of computation for \(n\) trials is given by \(\sum_{i=1}^n (c_1+c_2 i)\), which is quadratic in \(n\), the total number of trials. For \(\tilde{\eta} > 0\), from Theorem \ref{th:geometric}, each suboptimal arm has at most \(d(1+\tilde{\eta})\log n \) samples at time \(n\), for some constant \(d\), while the optimal arm has close to \( n \) samples. The number of batches allocated to each sub-optimal arm \(a\) till time \(n\), \(N_{B,a}(n)\), satisfies 
\[ (1+\tilde{\eta})^{N_{B,a}(n)+1} \le  d (1+\tilde{\eta}) \log(n) \implies  N_{B,a}(n) = O\lrp{\log\log(n)}.\] 
Similar computation gives that \(N_{B,1}(n)\), is \( O(\log(n)) \). 

The computational cost for the batches when the best arm won is at most \( (1+\tilde{\eta})^{N_{B,1}(n)-1} + N_{B,1}(n) O(\log(n)) \), which is \( O(n)\). The first term in the previous expression is the total cost of computing the index of arm \(1\) over the \(N_{B,1}(n)\) many batches in which arm \(1\) won, while the second term is an upper bound on that for the sub-optimal arms. This cost for the batches when sub-optimal arms win is at most \( N_{B,a}(n)K O(n) + O(\log(n))  \), where the first term is the computational cost of the index of arm 1, which has \(O(n)\) samples, contributing \(O(n\log\log(n))\) to the cost. Thus, the total worst-case computational cost of \(\KLinf \)-UCB with \( \tilde{\eta}  > 0\) is at most \( O(n\log\log(n)) \), where the multiplicative constant is given by $\frac{1}{\log\lrp{1+\tilde{\eta}}}$.

{\noindent \bf Computing \(\KLinf\)-UCB index: } Computing \(\KLinf\)-UCB index in the original form would require inverting the \(\KLinf\) functional. Since there is no closed form expression for \(\KLinf\), one approach for solving this can be by binary searching for the second argument, computing \(\KLinf\) at each iteration. However, representation of the index in (\ref{eq:index_reformulated}) gives it as a solution to a single optimization problem. This is discussed in Appendix~\ref{appx:index.compute}.

\section{Numerical results}\label{Sec:Numerics}

\begin{figure}[p]
\centering	\includegraphics[width=.6\textwidth]{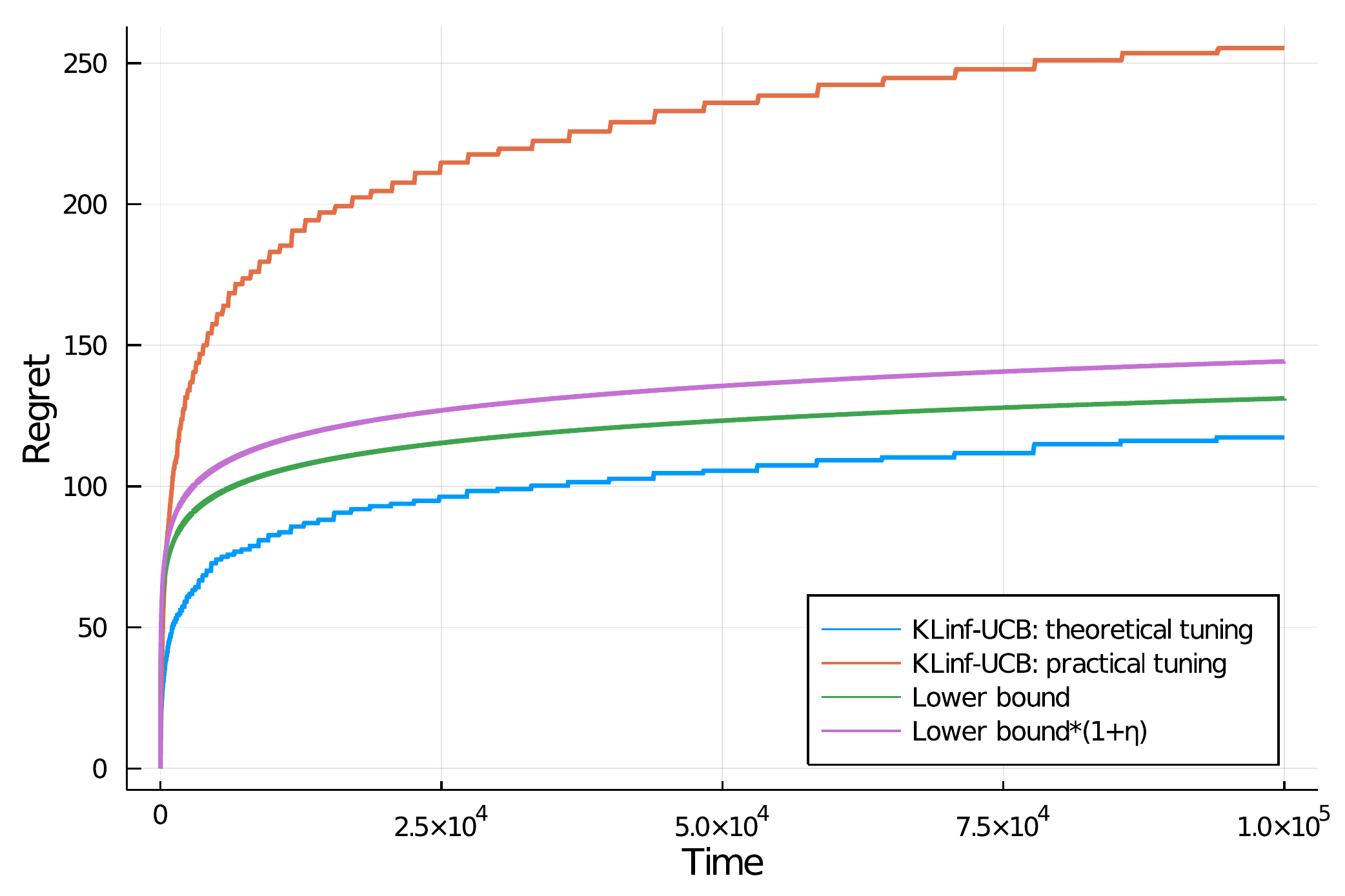}
	\caption{Comparison of regret of our algorithm with different thresholds, with batch-multiplicative factor set to \(\tilde\eta = 0.1\) in each, the asymptotic lower bound of \cite{BURNETAS1996} and the asymptotic upper bound in Theorem \ref{th:geometric}. Plots are averaged over 100 independent experiments. The batches are visible in the staircase pattern: long horizontal steps correspond to batches in which the optimal arm is chosen, and each vertical step represents a short batch where a suboptimal arm is picked.
        }\label{fig:KLinfUCBEasy}
\end{figure}

In this section we discuss the numerical studies undertaken to demonstrate the superiority of \(\KLinf\)-UCB. We run our algorithm on two different bandit problems. In both these experiments, the algorithm is presented with a  two-armed bandit, each arm having a Generalized Pareto (GenPar) distribution. It is a heavy-tailed distribution, which has \(3\) parameters, \( \mu, \sigma, \zeta \), which correspond to location, scale, and shape, respectively. When \(\zeta > 0\), it has a density function given by \( h(x) = \sigma\inv \lrp{1+\zeta \sigma\inv (x-\mu)}^{-1-\zeta\inv} \), for \(x \ge \mu \).

In the first experiment, we let  \(\epsilon=0.7\), and \(B\) is set to \(7\). Arm 1 is GenPar\( (-1,2,0.2) \), and arm 2 is GenPar\( (-1,1,0.2) \). Clearly, arm 1 is optimal with mean \( = 1.5 \), and whenever the algorithm chooses arm 2, it suffers a regret of \( 1.25 \). Moreover, \(\KLinf(\mu_2,1.5)\approx 0.1\).

We compare the performance of \(\KLinf \)-UCB with the theoretical threshold of Theorem \ref{th:geometric} and \(\KLinf\)-UCB with an agressive, practically tuneed threshold of \(\log t\) with the asymptotic lower bound. In Figure \ref{fig:KLinfUCBEasy}, we plot the regret incurred by both of these, along with the asymptotic lower bound, and 1.1 times the lower bound (which is the asymptotic upper-bound for \(\KLinf\)-UCB). 

As can be seen from Figure \ref{fig:KLinfUCBEasy}, algorithm with agressive threshold performs significantly better. It, in fact, suffers regret lower than the lower bound, even for large horizons. This is not surprising, and highlights the asymptotic nature of the lower bound. In \(\KLinf\)-UCB on the otherhand, the \(\log\log t\) term contributes significantly to the threshold over the horizon considered, leading to high regret. Henceforth, we only consider \(\KLinf\)-UCB with threshold set to \( \log(t) \).

We also compare the performance of the agressive algorithm with that of Robust-UCB with the truncation based estimator, proposed by \cite{bubeck2013heavytail}. Figure \ref{fig:EasyDifficultRatioComparison} plots the ratio of regret suffered by Robust-UCB and our algorithm, on two different bandit instances. ``Easy problem'' in the figure corresponds to the set-up of experiment \(1\), described above.

In the second experiment, we consider a slightly more difficult-to-learn setting, where the tails of the arm-distribution  are heavier than in the first experiment. We see  that both the algorithms suffer more regret on this problem compared to the previous bandit instance. However, the performance of Robust-UCB degrades much more. For this experiment, we set \(\epsilon\) to \(0.1 \), and let \( B = 13 \). The arm-distributions are set to GenPar\( (2.17,3.7,0.5) \) for arm \(1\), and GenPar\( (-1,2,0.71) \) for arm \(2\). In this setting, the optimal arm has mean \( 9.57 \), and when the sub-optimal arm is pulled, the algorithm suffers \(3.674\) units of regret.

Figure \ref{fig:EasyDifficultRatioComparison} shows that even with our batching, we significantly out-perform the Robust-UCB algorithm. It suffers retret that is 40 times that of \(\KLinf \)-UCB on the easy problem (setting of experiment 1), and the regret ratio is 19 the difficult setting of experiment 2.

\begin{figure}[p]
 \centering	\includegraphics[width=.6\textwidth]{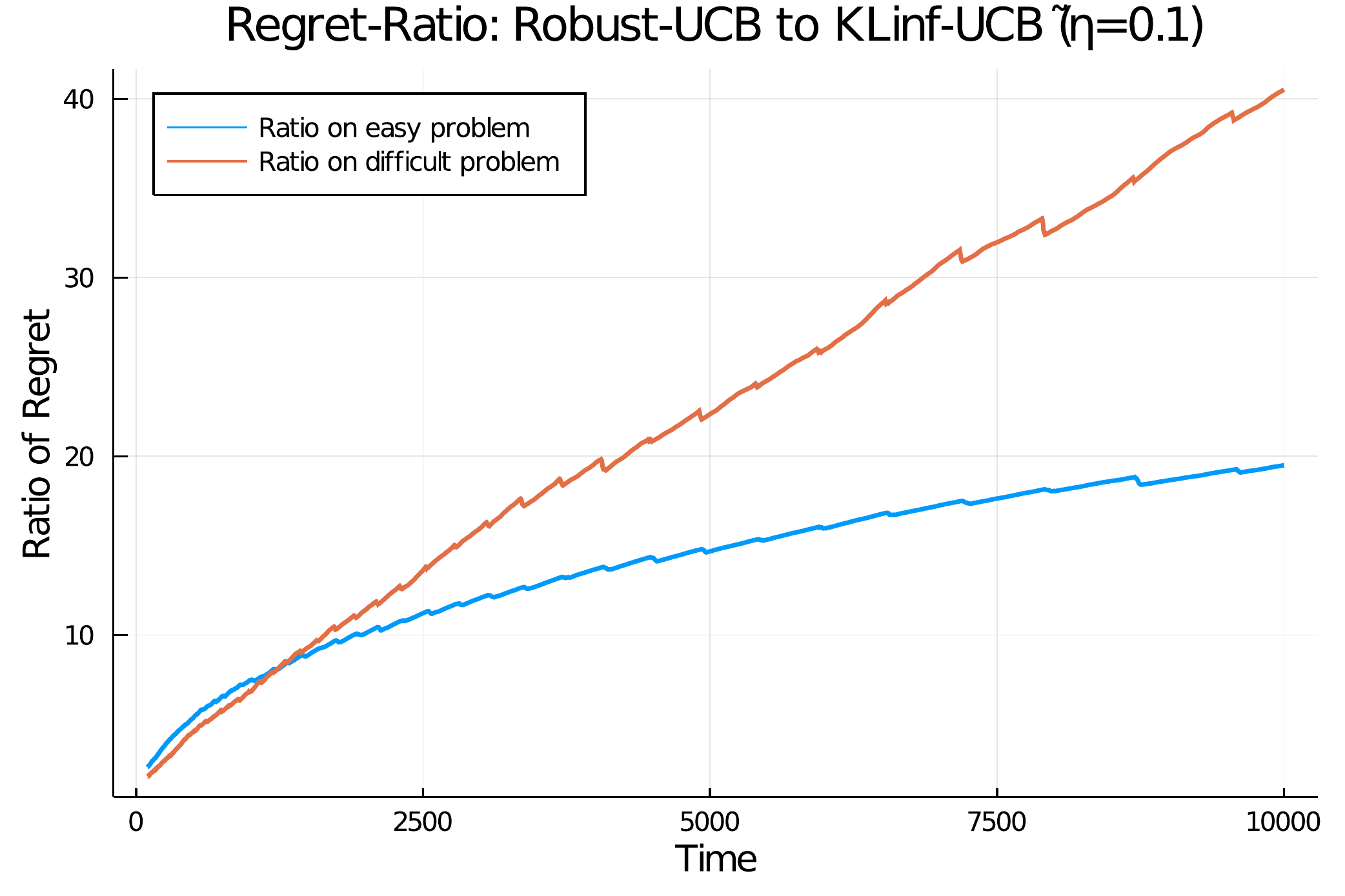}
	\caption{Ratio of regret of Robust-UCB with truncated empirical mean and \(\KLinf\)-UCB with batch-multiplicative factor set to \(\tilde\eta=0.1\), and with an aggressive threshold of \(\log t\), for simple and difficult bandit instances. This figure demonstrates that the performance of Robust-UCB degrades much more than that of our algorithm on a difficult bandit instance.}\label{fig:EasyDifficultRatioComparison}
\end{figure}

\section{Conclusion}\label{Sec:conclusions}

We consider minimising regret for heavy-tailed bandits. Our approach follows the UCB template. But instead of constructing upper confidence intervals, we ``let the lower bound speak'', ending up with $\KLinf$-based indices. We exploit the dual characterization of $\KLinf$ and the associated index, which guides both our index implementation and its statistical concentration analysis. We show that with these ingredients it is possible to match the instance-optimal regret lower bounds. Using a batched sampling scheme, we achieve total run time linear in the number of samples, at the cost of a small constant factor in the regret. We prove that $\KLinf$ indices dominate standard mean estimators for heavy-tailed distributions. We empirically validate that this translates into
much improved regret on synthetic problems.
Here are a few remarks to conclude the paper
\begin{itemize}
\item
  One may generalise the moment constraint $\Exp{\abs{X}^{1+\epsilon}}$ by requiring instead that $\Exp{f(X)} \le B$ for some convex (and super-linear) function $f$. To make this work, one would have to prove compactness of the corresponding region for $\vec\lambda$ in the dual formulation, so as to make the uniform-prior based regret analysis work. Or one would have to prove continuity of $\KLinf$ in its second argument and class parameters to employ the perturbation-based approach from Appendix \ref{app:perturbations}.

\item
  We see in experiments that performance is sensitive to the choice of threshold, and that our theoretically motivated thresholds are (currently) conservative. Our thresholds come from mixture martingales with universal coding/regret guarantees. Perhaps a redundancy-based analysis would be better suited for proving tighter concentration inequalities, and could reduce the threshold from $\ln n$ to $\ln \ln n$.

\item
  We considered the class of distributions with bounded \emph{uncentered} $(1+\epsilon)^\text{th}$ moment. A natural problem would be to extend the approach to the centred analogue.

\end{itemize}


\bibliography{BibTex}

\begin{thebibliography}{30}
\providecommand{\natexlab}[1]{#1}
\providecommand{\url}[1]{\texttt{#1}}
\expandafter\ifx\csname urlstyle\endcsname\relax
  \providecommand{\doi}[1]{doi: #1}\else
  \providecommand{\doi}{doi: \begingroup \urlstyle{rm}\Url}\fi

\bibitem[Agrawal(1995)]{agrawal1995sample}
Rajeev Agrawal.
\newblock Sample mean based index policies with o (log n) regret for the
  multi-armed bandit problem.
\newblock \emph{Advances in Applied Probability}, pages 1054--1078, 1995.

\bibitem[Agrawal and Goyal(2012)]{SAgrawal2011TS1}
Shipra Agrawal and Navin Goyal.
\newblock Analysis of thompson sampling for the multi-armed bandit problem.
\newblock In \emph{Conference on learning theory}, pages 39--1. JMLR Workshop
  and Conference Proceedings, 2012.

\bibitem[Agrawal et~al.(2020{\natexlab{a}})Agrawal, Juneja, and
  Glynn]{pmlr-v117-agrawal20a}
Shubhada Agrawal, Sandeep Juneja, and Peter Glynn.
\newblock Optimal $\delta$-correct best-arm selection for heavy-tailed
  distributions.
\newblock In \emph{Proceedings of the 31st International Conference on
  Algorithmic Learning Theory}, volume 117 of \emph{Proceedings of Machine
  Learning Research}, pages 61--110. PMLR, 08 Feb--11 Feb 2020{\natexlab{a}}.

\bibitem[Agrawal et~al.(2020{\natexlab{b}})Agrawal, Koolen, and
  Juneja]{agrawal2020optimal}
Shubhada Agrawal, Wouter~M Koolen, and Sandeep Juneja.
\newblock Optimal best-arm identification methods for tail-risk measures.
\newblock \emph{arXiv preprint arXiv:2008.07606}, 2020{\natexlab{b}}.

\bibitem[Auer et~al.(2002)Auer, Cesa-Bianchi, and Fischer]{Auer2002}
Peter Auer, Nicol{\`o} Cesa-Bianchi, and Paul Fischer.
\newblock Finite-time analysis of the multiarmed bandit problem.
\newblock \emph{Machine Learning}, 47\penalty0 (2):\penalty0 235--256, May
  2002.

\bibitem[Berge(1997)]{berge1997topological}
C.~Berge.
\newblock \emph{Topological Spaces: Including a Treatment of Multi-valued
  Functions, Vector Spaces, and Convexity}.
\newblock Dover books on mathematics. Dover Publications, 1997.
\newblock ISBN 9780486696539.

\bibitem[Billingsley(2013)]{billingsley2013convergence}
P.~Billingsley.
\newblock \emph{Convergence of Probability Measures}.
\newblock Wiley Series in Probability and Statistics. Wiley, 2013.

\bibitem[Bubeck et~al.(2012)Bubeck, Cesa-Bianchi, et~al.]{bubeck2012Survey}
S{\'e}bastien Bubeck, Nicolo Cesa-Bianchi, et~al.
\newblock Regret analysis of stochastic and nonstochastic multi-armed bandit
  problems.
\newblock \emph{Foundations and Trends{\textregistered} in Machine Learning},
  5\penalty0 (1):\penalty0 1--122, 2012.

\bibitem[Bubeck et~al.(2013)Bubeck, Cesa-Bianchi, and
  Lugosi]{bubeck2013heavytail}
S{\'e}bastien Bubeck, Nicolo Cesa-Bianchi, and G{\'a}bor Lugosi.
\newblock Bandits with heavy tails.
\newblock \emph{IEEE Transactions on Information Theory}, 59\penalty0
  (11):\penalty0 7711--7717, 2013.

\bibitem[Burnetas and Katehakis(1996)]{BURNETAS1996}
Apostolos~N. Burnetas and Michael~N. Katehakis.
\newblock Optimal adaptive policies for sequential allocation problems.
\newblock \emph{Advances in Applied Mathematics}, 17\penalty0 (2):\penalty0 122
  -- 142, 1996.

\bibitem[Capp{\'e} et~al.(2013{\natexlab{a}})Capp{\'e}, Garivier, Maillard,
  Munos, and Stoltz]{cappe2013kullbackSupp}
O~Capp{\'e}, A~Garivier, OA~Maillard, R~Munos, and G~Stoltz.
\newblock Kullback--leibler upper confidence bounds/supplemental article.
\newblock \emph{Ann. Stat}, 41\penalty0 (3):\penalty0 1516--1541,
  2013{\natexlab{a}}.

\bibitem[Capp{\'e} et~al.(2013{\natexlab{b}})Capp{\'e}, Garivier, Maillard,
  Munos, Stoltz, et~al.]{cappe2013kullback}
Olivier Capp{\'e}, Aur{\'e}lien Garivier, Odalric-Ambrym Maillard, R{\'e}mi
  Munos, Gilles Stoltz, et~al.
\newblock Kullback--{L}eibler upper confidence bounds for optimal sequential
  allocation.
\newblock \emph{The Annals of Statistics}, 41\penalty0 (3):\penalty0
  1516--1541, 2013{\natexlab{b}}.

\bibitem[Cowan and Katehakis(2015)]{cowan2015asymptotically}
Wesley Cowan and Michael~N Katehakis.
\newblock An asymptotically optimal policy for uniform bandits of unknown
  support.
\newblock \emph{arXiv preprint arXiv:1505.01918}, 2015.

\bibitem[Cowan et~al.(2018)Cowan, Honda, and Katehakis]{Cowan2018}
Wesley Cowan, Junya Honda, and Michael~N. Katehakis.
\newblock Normal bandits of unknown means and variances.
\newblock \emph{Journal of Machine Learning Research}, 18\penalty0
  (154):\penalty0 1--28, 2018.

\bibitem[Dembo and Zeitouni(2010)]{dembo2010large}
Amir Dembo and Ofer Zeitouni.
\newblock Large deviations techniques and applications. corrected reprint of
  the second (1998) edition. stochastic modelling and applied probability, 38,
  2010.

\bibitem[Garivier and Capp{\'e}(2011)]{garivier2011kl}
Aur{\'e}lien Garivier and Olivier Capp{\'e}.
\newblock The kl-ucb algorithm for bounded stochastic bandits and beyond.
\newblock In \emph{Proceedings of the 24th annual conference on learning
  theory}, pages 359--376, 2011.

\bibitem[Honda and Takemura(2010)]{HondaBounded10}
Junya Honda and Akimichi Takemura.
\newblock An asymptotically optimal bandit algorithm for bounded support
  models.
\newblock In \emph{In Proceedings of the Twenty-third Conference on Learning
  Theory (COLT 2010}, pages 67--79. Omnipress, 2010.

\bibitem[Honda and Takemura(2011)]{honda2011asymptotically}
Junya Honda and Akimichi Takemura.
\newblock An asymptotically optimal policy for finite support models in the
  multiarmed bandit problem.
\newblock \emph{Machine Learning}, 85\penalty0 (3):\penalty0 361--391, 2011.

\bibitem[Honda and Takemura(2015)]{honda2015SemiBounded}
Junya Honda and Akimichi Takemura.
\newblock Non-asymptotic analysis of a new bandit algorithm for semi-bounded
  rewards.
\newblock \emph{The Journal of Machine Learning Research}, 16\penalty0
  (1):\penalty0 3721--3756, 2015.

\bibitem[Kaufmann et~al.(2012)Kaufmann, Korda, and Munos]{kaufmann2012TS}
Emilie Kaufmann, Nathaniel Korda, and R{\'e}mi Munos.
\newblock Thompson sampling: An asymptotically optimal finite-time analysis.
\newblock In \emph{International Conference on Algorithmic Learning Theory},
  pages 199--213. Springer, 2012.

\bibitem[Lai and Robbins(1985)]{LaiRobbins85}
T.L Lai and Herbert Robbins.
\newblock Asymptotically efficient adaptive allocation rules.
\newblock \emph{Advances in Applied Mathematics}, 6\penalty0 (1):\penalty0 4 --
  22, 1985.

\bibitem[Lattimore(2017)]{lattimore2017scale}
Tor Lattimore.
\newblock A scale free algorithm for stochastic bandits with bounded kurtosis.
\newblock In \emph{Advances in Neural Information Processing Systems}, pages
  1584--1593, 2017.

\bibitem[Lattimore and Szepesv{\'a}ri(2020)]{lattimore2020bandit}
Tor Lattimore and Csaba Szepesv{\'a}ri.
\newblock \emph{Bandit algorithms}.
\newblock Cambridge University Press, 2020.

\bibitem[Lugosi and Mendelson(2019)]{lugosi2019mean}
G{\'a}bor Lugosi and Shahar Mendelson.
\newblock Mean estimation and regression under heavy-tailed distributions: A
  survey.
\newblock \emph{Foundations of Computational Mathematics}, 19\penalty0
  (5):\penalty0 1145--1190, 2019.

\bibitem[Maillard et~al.(2011)Maillard, Munos, and Stoltz]{maillard2011finite}
Odalric-Ambrym Maillard, R{\'e}mi Munos, and Gilles Stoltz.
\newblock A finite-time analysis of multi-armed bandits problems with
  kullback-leibler divergences.
\newblock In \emph{Proceedings of the 24th annual Conference On Learning
  Theory}, pages 497--514, 2011.

\bibitem[Posner(1975)]{posner1975random}
E.~Posner.
\newblock Random coding strategies for minimum entropy.
\newblock \emph{IEEE Transactions on Information Theory}, 21\penalty0
  (4):\penalty0 388--391, 1975.

\bibitem[Robbins(1952)]{robbins1952}
Herbert Robbins.
\newblock Some aspects of the sequential design of experiments.
\newblock \emph{Bull. Amer. Math. Soc.}, 58\penalty0 (5):\penalty0 527--535, 09
  1952.

\bibitem[Sundaram(1996)]{SundaramOpt1996}
Rangarajan~K. Sundaram.
\newblock \emph{A First Course in Optimization Theory}.
\newblock Cambridge University Press, 1996.

\bibitem[Thompson(1933)]{Thomp1933}
William~R. Thompson.
\newblock On the likelihood that one unknown probability exceeds another in
  view of the evidence of two samples.
\newblock \emph{Biometrika}, 25\penalty0 (3/4):\penalty0 285--294, 1933.

\bibitem[Vakili et~al.(2013)Vakili, Liu, and Zhao]{vakili2013deterministic}
Sattar Vakili, Keqin Liu, and Qing Zhao.
\newblock Deterministic sequencing of exploration and exploitation for
  multi-armed bandit problems.
\newblock \emph{IEEE Journal of Selected Topics in Signal Processing},
  7\penalty0 (5):\penalty0 759--767, 2013.

\end{thebibliography}
\appendix

\section{\(\KLinf\) - dual formulation and properties}\label{App:KLinfProp}

\paragraph{L\'evy metric and weak convergence:} The L\'evy metric, \(d_L(\kappa,\eta)\), between probability distributions (see, \citet[Appendix D]{dembo2010large}) $\kappa$ and $\eta$ on $\Re$, is given by 
$$d_L(\eta,\kappa)=\inf\lrset{\zeta > 0 : F_\kappa(x-\zeta)-\zeta \leq F_\eta(x) \leq F_\kappa(x + \zeta) + \zeta, ~\forall x\in\Re  } ,$$ 
where for \(\zeta \in \mathcal P(\Re) \), \( F_\zeta \) denotes the CDF function for \( \zeta \), i.e., for \(x\in \Re\),  \(F_\zeta(x) := \zeta( -\infty,x ]\). Furthermore, weak convergence of sequences of probability measures is equivalent to their convergence in the L\'evy metric (see, \citet[Theorem 6.8]{billingsley2013convergence}, and  \citet[Theorem D.8]{dembo2010large}).

For \(B>0\) and \(\epsilon > 0\), let \(\mathcal L_B = \lrset{ \kappa\in\mathcal{P}(\Re) : \E{\kappa}{f\lrp{X}} \leq B} \), where for \(y\in\Re\), \(f(y) = \abs{y}^{1+\epsilon}\). Let \(x\in\Re\) be such that \(f\lrp{x}  < B\), and for \(\eta \in \mathcal{P}(\Re)\), recall that \( m(\eta):=\E{\eta}{X} \), and  
\[\KLinfR{\mathcal L_B}(\eta,x) = \inf\lrset{\KL(\eta, \kappa): \kappa\in \mathcal{L}_B,~ m(\kappa)\geq x} .\]
In the rest of the appendix, we denote \( \KLinfR{\mathcal L_B}(\eta, x)  \) by \(\KLinf(\eta,x) \). Furthermore, let \(\Supp(\eta)\) denote the support of measure \(\eta\). We first review an alternative representation and properties of the function \(\KLinf\), which will be useful in the analysis of our algorithm. The following Lemma is due to \citet[Theorem 12]{pmlr-v117-agrawal20a}, which we state for completeness. 

\begin{lemma}[Dual formulation of \( \KLinf\)]\label{lem:DualFormulation} 
	For \(\eta \in \mathcal{P}(\Re)\) and \(x\) such that  \( \abs{x}^{1+\epsilon} < B\),
	\begin{equation}
	\label{eq:KLinf}
	\KLinf(\eta,x) = \max\limits_{(\lambda_1,\lambda_2)\in\mathcal{R}(x,B)}~ \mathbb{E}_{\eta}\lrp{\log\lrp{1-(X-x)\lambda_1 -(B-\abs{X}^{1+\epsilon})\lambda_2}},
	\end{equation}
	where $$\mathcal{R}(x,B) = \lrset{(\lambda_1 \ge 0,\lambda_2\ge 0):  \frac{\epsilon\lambda^{1+{\frac{1}{\epsilon}}}_1}{\lambda^{\frac{1}{\epsilon}}_2(1+\epsilon)^{1+\frac{1}{\epsilon}}} + B\lambda_2 - x \lambda_1 -1 \le 0 }.$$ There is a unique \( \lrp{\lambda^*_1,\lambda^*_2} \) that achieves the maximum above. Furthermore, any primal variable that achieves infimum in (\ref{eq:KLinf}), satisfies 
	\[ \frac{d\kappa^*}{d\eta}(y) = \lrp{1-(y-x)\lambda^*_1 - (B-f(y))\lambda^*_2}\inv,\quad\text{ for }y\in \Supp(\eta), \]
	\( \abs{\lrset{\Supp(\kappa^*) \setminus \Supp(\eta)} } \) is at most 1, and for \(y^*\in\lrset{\Supp(\kappa^*)\setminus \Supp(\eta)} \), \( 1-(y^*-x)\lambda^*_1 - (B-f(y^*))\lambda_2 = 0 \). 
\end{lemma}

\begin{lemma}\label{rem:compactness_of_R}
	The region \( \mathcal R(x,B) \) defined in Lemma \ref{lem:DualFormulation} is compact whenever \( \abs{x}^{1+\epsilon} < B  \). It satisfies \[ \lambda_1 \le \frac{1}{B^{\frac{1}{1+\epsilon}} -{x}}, \quad \text{ and }\quad  \lambda_2 \le \frac{1}{B-\abs{x}^{1+\epsilon}}.\]
\end{lemma}

\begin{proof}
	Let \[ h(\lambda_1,\lambda_2) = \epsilon\lambda^{1+{\frac{1}{\epsilon}}} + B\lambda_2 \lrp{\lambda^{1+\frac{1}{\epsilon}}_2(1+\epsilon)^{1+\frac{1}{\epsilon}}} - (x \lambda_1 + 1)\lrp{\lambda^{\frac{1}{\epsilon}}_2(1+\epsilon)^{1+\frac{1}{\epsilon}}}. \]
	Then the constraint in \( \mathcal R(x,B) \) is \( h(\lambda_1,\lambda_2) \le 0 \). Clearly, \( h \) is convex in  \( \lambda_1 \). Thus, the given constraint implies \(  \min_{\lambda_1 \ge 0} h(\lambda_1,\lambda_2) \le 0 \). This gives the bound on \( \lambda_2 \). Similarly, \( \min_{\lambda_2\ge 0} h(\lambda_1,\lambda_2) \le 0 \) gives the bound on \( \lambda_1 \). 
\end{proof}

For a fixed \(x\), \(\KLinf(\cdot,x) \) is a function from \(\mathcal{P}(\Re) \) to \(\Re^+\). Recall that we endow the space \(\mathcal{P}(\Re)\) with the the topology of weak convergence, or equivalently with the L\'evy metric (see, e.g., \citet[Appendix D]{dembo2010large} for definitions of the two topologies and their equivalence).

\begin{lemma}[Properties of \( \mathcal{L}_B \) and \( \mathcal{R}(x,B) \)] \label{lem:PropLR}
	\( \mathcal{L} \) is a a uniformly integrable  collection, and a compact set of probability measures in the topology of weak convergence. Furthermore, \( \mathcal{R}(x,B) \) satisfies
	\begin{enumerate}
		\item  \label{prop1:Rx} for \(B> 0\), it is an upper-hemicontinuous function of \(x\) on \((-f\inv(B),f\inv(B))\),
		\item  \label{prop1:RB} for \( x \in  \Re\), it is an upper-hemicontinuous function of \(B\) on \( (f\lrp{x},\infty) \). 
	\end{enumerate}
\end{lemma}

\begin{proof}
	Uniform integrability and compactness of \( \mathcal{L}_B \) follow from \cite[Lemma 3.2]{agrawal2020optimal}. 
	
	Proof of \ref{prop1:Rx}: 	To see upper-hemicontinuity of \( \mathcal{R}(x,B) \) for a fixed \(B\), consider a sequence \( x_n \rightarrow x\). Let \( \eta_n \in \mathcal{R}(x_n,B) \) be a sequence of measures in \( \mathcal{ L}_B \), which is a tight collection of probability measures in the weak topology. Then, there is a subsequence \( \eta_{p_{i}} \) which converges weakly to \(\eta_{p}\) (see, \cite{billingsley2013convergence}). From \cite[Proposition 9.8]{SundaramOpt1996}, it is sufficient to show that \( \eta_p  \) belongs to \( \mathcal{R}(x,B) \). Clearly, \( \eta_p \) belongs to class \(\mathcal{ L}_B \) since it is a closed set, and \( \eta_{p_i} \) belong to \( \mathcal{ L}_B \). Furthermore, since \( \mathcal{ L}_B \) is a uniformly integrable collection, \( \eta_n \xRightarrow{D}\eta \) implies that \( \E{\eta_{p_i}}{X} \rightarrow \E{\eta_p}{X}\), and hence, \( \E{\eta_p}{X} \geq x \). 
	
	Proof of \ref{prop1:RB}: 	To see upper-hemicontinuity of \( \mathcal{R}(x,B) \) for a fixed \(x\), consider a sequence \( B_n \rightarrow B \), and let \(\eta_n\) be a sequence in \(\mathcal{R}(x,B_n)\). For any fixed \( \delta >0  \), there exists  \(n_0\) such that for all \(n \geq n_0\), \(B_n \leq B + \delta \), and the sequence \( \lrset{\eta_n} \), for all \(n\geq n_0\), satisfies \( \E{\eta_n}{f\lrp{X}} \leq B + \delta \). As above, since \( \mathcal{L}_{B+\delta} \) is a closed, tight, and uniformly integrable collection of probability measures, arguments as in the previous paragraph show that \( \E{\eta_{p_i}}{X} \rightarrow \E{\eta_p}{X} \geq x \), and \( \eta_p  \) belongs to \( \mathcal{L}_{B+\delta} \). Since \( \delta \) was arbitrary, \( \eta_p  \) belongs to \( \mathcal{ L}_B \), and hence to \( \mathcal{R}(x,B) \). 
\end{proof}

\begin{lemma}[Properties of \(\KLinf \)]\label{lem:PropKLinf}
	For a fixed \(\eta \in\mathcal{P}(\Re) \) and 
	\begin{enumerate}
		\item \label{prop:contInx} for a fixed \( B > 0 \), \( \KLinf(\eta,x) \) is a continuous function of \(x\) on \((-f^{-1}(B),f^{-1}(B)) \). 
		\item \label{prop:contInB} for a fixed \( x \in\Re\), \( \KLinf(\eta,x) \) is a continuous function of \(B\) on \( (f({x}),\infty)\).
	\end{enumerate}
	
\end{lemma}
\begin{proof}
	To prove the continuity in \(x\) and \(B\), we prove lower- and upper-semicontinuity separately. To prove \ref{prop:contInx}, we first argue that for a fixed \(B\), \( \KLinf(\eta,x) \) is a convex function of \(x\), and hence an upper-semicontinuous function. To see the convexity, consider \( x_1, x_2 \in (-f^{-1}(B), f^{-1}(B)) \). Let 
	\[\mathcal{R}(x,B) := \lrset{\gamma\in\mathcal{P}(\Re): ~ \E{\gamma}{X} \geq x, ~ \E{\gamma}{f\lrp{X}} \leq B }. \]
	Clearly, the set \(\mathcal{R}(x,B)  \) is convex and non-empty for the choices of \(x\) and \(B\) under consideration. Let \( \kappa_1, \kappa_2 \in \mathcal{R}(x,B) \) be such that \( \KLinf(\eta,x_1) = \KL(\eta,\kappa_1) \) and \( \KLinf(\eta,x_2) = \KL(\eta,\kappa_2) \). Existence of \(\kappa_1 \) and \(\kappa_2\) is guaranteed by compactness of \(\mathcal{R}(x,B)\) and lower-semicontinuity of \(\KL\). Furthermore, for \(\lambda\in (0,1)\), \[ R_{12}(\lambda) := \lrset{\lambda \mathcal{R}(x_1,B)+(1-\lambda)\mathcal{R}(x_2,B)} \subset \mathcal{R}(\lambda x_1 + (1-\lambda)x_2,B).\] 
	Consider the following inequalities: 
	\begin{align*} 
	\KLinf(\eta,\lambda x_1+(1-\lambda)x_2) ~\leq~ \inf\limits_{\kappa\in R_{12}(\lambda)}~ \KL(\eta,\kappa) ~\leq~ \KL(\eta,\lambda\kappa_1+(1-\lambda)\kappa_2).
	\end{align*}
	Using joint convexity of \( \KL \), the above can further be bounded from above by $\lambda\KL(\eta,\kappa_1) + (1-\lambda) \KL(\eta,\kappa_2)$, which equals $\lambda\KLinf(\eta,x_1)+(1-\lambda)\KLinf(\eta,x_2)$ by choice of \( \kappa_1 \) and \( \kappa_2 \), giving 
	\[ \KLinf(\eta,\lambda x_1 + (1-\lambda) x_2) \leq \lambda \KLinf(\eta,x_1) + (1-\lambda) \KLinf(\eta,x_2) .\]

	Convexity in \(B\) also follows similarly, proving the upper-semicontinuity of \( \KLinf \) in \(x\) and in \(B\), under the given conditions.  
	
	For \(\eta\in\mathcal{P}(\Re)\), since \( \KL(\eta,\cdot) \) is a lower-semicontinuous function in the topology of weak convergence (see, \cite{posner1975random}), and the region of optimization, \(\mathcal{R}(x,B)\), is a non-empty, compact, upper-hemicontinuous correspondence of \(x\) and \(B\) under the respective given conditions (Lemma \ref{lem:PropLR}), the optimal value, \( \KLinf(\eta,x) \) is lower-semicontinuous in \(x\) for a fixed \(B\), and lower-semicontinuous in \(B\) for a fixed \(x\) (see, \cite[Theorem 2, page 116]{berge1997topological}).
\end{proof}

\section{Results related to regret guarantees}\label{App:RegretGuarantees}
In this appendix, we state and prove the results that assist us in the proof of Theorem \ref{th:geometric}. 

\subsection{Proof of Lemma \ref{lem:BoundOfE2j}}\label{proof:lem:BoundOfE2j}

Let us first consider \( \mu_b = \delta_{-B^{1/(1+\epsilon)}} \). In this case, the statement is vacuously true since \( \hat{\mu}_{b,s} = \mu_b\). 

Now, let \(\mu_b \ne \delta_{-B^{1/(1+\epsilon)}}\). For \(\epsilon > 0\) and \(B> 0\) let \(f(y) = \abs{y}^{1+\epsilon} \). For \(c\in [0,B] \) define \(f\inv(c) = c^{1/(1+\epsilon)}\). Using the dual formulation for \(\KLinf(\hat{\mu}_{b,s},{m(\mu_1)})\) from Lemma \ref{lem:DualFormulation}, the required probability equals
\[ \mathbb{P}\lrp{ \max\limits_{\bm{\lambda} \in \mathcal{R}({m(\mu_1)},B)}~ \frac{1}{s} \sum\limits_{i=1}^s  \log\lrp{1-\lrp{X_i-{m(\mu_1)}} \lambda_1 - \lrp{B - f\lrp{X_i}} \lambda_2} \leq \KLinf\lrp{\mu_b,{m(\mu_1)}} -\delta },\]
which is bounded from above by 
\[ \mathbb{P}\lrp{ \frac{1}{s}~ \sum\limits_{i=1}^s  \log\lrp{1-\lrp{X_i-{m(\mu_1)}} \lambda^*_1 - \lrp{B - f\lrp{X_i}} \lambda^*_2} \leq \KLinf\lrp{\mu_b,{m(\mu_1)}} -\delta }, \]
where 
\[ (\lambda^*_1,\lambda^*_2) \in \argmax\limits_{\bm{\lambda}\in\mathcal{R}({m(\mu_1)},B)}~ \E{\mu_b}{\log\lrp{1-(X-{m(\mu_1)})\lambda_1 -\lrp{B -f\lrp{X} }\lambda_2 }} .\]
For \(\theta \geq 0\), the required probability can be bounded by 
\[ \mathbb{P}\lrp{ -\theta \sum\limits_{i=1}^s  \log\lrp{1-\lrp{X_i-{m(\mu_1)}} \lambda^*_1 - \lrp{B - f\lrp{X_i}} \lambda^*_2} \geq -s\theta\lrp{\KLinf\lrp{\mu_b,{m(\mu_1)}} -\delta} },\]
which can further be bounded by 
\[  \E{\mu_b}{\exp\lrset{-\theta\sum\limits_{i=1}^s \log\lrp{ 1-(X_i-{m(\mu_1)})\lambda^*_1 -\lrp{B-f\lrp{X_i}}\lambda^*_2 }}} e^{s\theta\lrp{\KLinf\lrp{\mu_b,{m(\mu_1)}}-\delta}}.  \]

Let \( Y_i := \log\lrp{ 1-(X_i-{m(\mu_1)})\lambda^*_1 -\lrp{B-f\lrp{X_i}}\lambda^*_2 } \).  Since \(X_i\sim \mu_b\) are i.i.d., \(Y_i\) are i.i.d. Furthermore, let \(Y\) be independent and identically distributed as \(Y_i\), and for \(\gamma \leq 1\), let \( \Lambda_{Y} := \log \Exp{ \exp\lrset{\gamma Y} }\). Obsesrve that \(\E{\mu_b}{Y} = \KLinf\lrp{\mu_b,{m(\mu_1)}}\). Then the above expression equals
\[ \exp\lrset{s\lrp{\theta\lrp{\KLinf\lrp{\mu_b,{m(\mu_1)}}-\delta} + \Lambda_{Y}(-\theta) }}.\]
Since the previous bound is true for all values of \( \theta \geq 0 \), in particular, we have that the required probability is bounded by 
\[\exp\lrset{-s~\sup\limits_{\theta\leq 0}~\lrset{\theta\lrp{\KLinf\lrp{\mu_b,{m(\mu_1)}} -\delta} - \Lambda_{Y}(\theta) } }.\]
Obsesrve that \( \Lambda_Y(0) = 0 \) and \(\Lambda_{Y}(-1) \leq 0\). To see the latter, 
\[ \Lambda_{Y}(-1) = \log\E{\mu_b}{\frac{1}{1-(X-{m(\mu_1)}) \lambda^*_1 -(B -f\lrp{X} ) \lambda^*_2  }} = \log \kappa_b(\Supp(\mu_b)) \le 0,  \]
where \( \kappa_b \) is the optimal primal variable for the \( \KLinf\lrp{\mu_b,{m(\mu_1)}} \) optimization problem, and satisfies \(\kappa_b(\Supp(\mu_b)) \le 1 \) (see Lemma \ref{lem:DualFormulation}). Furthermore, \( \Lambda_{Y}(\gamma) \) is a convex function of \( \gamma\). Whence, supremum in \[ \sup\limits_{\theta\le 0} \lrset{\theta(\KLinf(\mu_b,m(\mu_1))-\delta ) - \Lambda_{Y}(\theta)} \] is attained at some \( \theta^*< 0, \) and the optimal value equals \(I_Y(\KLinf(\mu_b,m(\mu_1)-\delta))\), where  
\[ I_Y(\KLinf(\mu_b,m(\mu_1))) := \sup\limits_{\theta} \lrset{\theta(\KLinf(\mu_b,m(\mu_1))-\delta ) - \Lambda_{Y}(\theta)}\]
is the large deviations rate function for the random variable \(Y\). Lemma \ref{lem:LDP} below shows that there exists \( \delta_0 >0\) and a constant \(c_b\), such that for all \( \delta < \delta_0 \), \( I_Y(\KLinf(\mu_b,m(\mu_1))-\delta) \ge c_b\delta^2 \). 

Since \( \delta \in [0,\min_{a\ne 1} \KLinf(\mu_a,m(\mu_1)) ] \), by convexity of \( I_Y \), there exists a constant \(c'_b\) (possibly smaller than \(c_b\)), such that \( I_Y(\KLinf(\mu_b,m(\mu_1))-\delta) \ge c'_b\delta^2. \)

Then, \(  c_\mu := \min_{a\ne 1} c'_a \). Clearly, \( I_Y(\KLinf(\mu_b,m(\mu_1))-\delta ) \ge c_\mu\delta^2 \), giving the desired bound. $\quad\quad \Box$

Given a random variable \(X\) with \( \Lambda_X(\theta) < \infty \) for \( \theta \in [-1,1] \). Let \( I(x) := \sup_\theta  \lrset{\theta x - \Lambda_X(\theta) }  \) denote the associated large deviation rate function. Let \( m(X) > 0 \), i.e., \( \Lambda'_X(0) > 0 \), and hence \( I(m(X)) = 0 \). 

Note that \( Y  \) defined above satisfies these conditions. 

\begin{lemma}\label{lem:LDP}
	For \(X\) satisfying the conditions above, \( \exists \delta_0 > 0 \) such that \( \forall \delta \in [0,\delta_0] \), there exists a constant \(c > 0\) such that \( I(m(X)-\delta) \ge c  \delta^2 \). 
\end{lemma}
\begin{proof}
	Observe that if \( \Lambda''_X(0) = 0  \), then \(X\) is degenerate, giving \(I(x) = \infty \) for \(x \ne m(X) \). 
	
	Next, consider a non-degenerate \(X\), meaning \( \Lambda''_X(0) > 0 \). Also, consider the \( \theta_y \) that satisfies \( \Lambda'_X(\theta_y) = y \). Here, \(\theta_{m(X)} = 0 \). Also, \( \Lambda'_X(\theta) \) is continuously differentiable for \( \theta \in (-1,1) \). By Implicit Function Theorem, there exists a neighbourhood \( (m(X)-
	\tilde{\delta}_0, m(X)) \) such that  for \(y\in (m(X)-\tilde{\delta}_0,m(X))\), \(\theta_y\) exists and is continuously differentiable. 
	
	Also, \( \theta_y  < 0 \) for \( y < m(X)\). Choose \( \tilde{\delta}_0 \) such that 
	\[ y\in (m(X) - \tilde{\delta}_0, m(X)] \implies -1 < \theta_y \le 0. \]
	
	The above discussion implies \( \Lambda''_X(\theta_y) \) is a continuous function of \(y\) for all \( y \in (m(X)-\tilde{\delta}_0, m(X)] \). Furthermore, for \(y\) in this range, \( \Lambda'_X(\theta_y) < \infty\) and \( \Lambda''_X(0) > 0 \). This gives \( \infty > \Lambda''_X(\theta_y) > 0 \). This gives that \( c := \sup\lrset{\Lambda''_X(\theta_y): ~y\in[m(X)-\tilde{\delta}_0, m(X)] } \) is finite (continuous function on a compact set). 
	
	Now, it can be checked that for \(y\in [m(X)-\tilde{\delta}_0, m(X)] \), \( I'(y) = \theta_y \) and \( I''(y) = \frac{1}{\Lambda''(\theta_y)} \ge c\inv \).
	
	Moreover, \( I(m(X)-\delta) = I(m(X)) + I'(m(X)) \delta + \delta^2/2 I''(\tilde{x}) \), for \( \tilde{x}\in [m(X)-\delta,m(X)] \). This gives 
	\[ I(m(X)-\delta) \ge \frac{\delta^2}{2c}, \quad \text{ for } \delta \le \delta_0. \]
\end{proof}

Recall that for \(T > K\) and for sub-optimal arm \(a\), 
\[E_{2j} = \mathbbm{1}\lrp{\KLinf( \hat{\mu}_a(T_j), m) \le \KLinf(\mu_a,m) -\delta, ~ A_{T_j} = a} , \]
where \( T_j \) denotes the random time marking the beginning of \(j^{th}\) batch. Let \( N_{B,a} \) denote the random number of batches allocated to arm \(a\) till time \(T\), N denote the total number of batches allocated till time \(T\), and \( B_{t}\) denote the size of the batch begining at time \(t\). 

\begin{lemma}\label{lem:SumE2jIntermediate}
	For \( T > K \), for sub-optimal arm a,
	\[\Exp{\sum\limits_{j=1}^N B_j E_{2j}} \le 
	\begin{cases} 
		&\frac{1+\tilde{\eta}}{c_\mu\delta^2}\lrp{\frac{1}{\log(1+\tilde{\eta}) } +\frac{1}{e} }, \text{ for } \tilde{\eta} > 0\\
		& \frac{1}{c_\mu\delta^2} + 1, \text{ otherwise.}
		\end{cases}
	\]
\end{lemma}
\begin{proof}
	Recall that 
	\[ \Exp{\sum\limits_{j=1}^N B_j E_{2j}} = \Exp{\sum\limits_{k=2}^{N_{B,a}}  B_{T^k_a} \mathbbm{1}\lrp{\KLinf( \hat{\mu}_a(T^k_a), m) \le \KLinf(\mu_a,m) -\delta}}. \]
	
	Let us first consider the case when \( \tilde{\eta} > 0 \). In this case, \( N_{B,a}  \) is at most \( \frac{\log(T)}{\log(1+\tilde{\eta})} \) and \( B_{T^k_a} \) is at most \(\tilde{\eta} N_a(T^k_a)+1 \), which in turn is at most \((1+\tilde{\eta})^{k-1}\). Thus, the required expectation is at most 
	
	\[ \sum\limits_{k=2}^{\frac{\log(T)}{\log(1+\tilde{\eta})}+1}  (1+\tilde{\eta})^{k-1} \mathbbm{P}\lrp{\KLinf( \hat{\mu}_a(T^k_a), m) \le \KLinf(\mu_a,m) -\delta}. \]
	
	Clearly, \( N_a(T^k_a) \) is deterministic.	Lemma \ref{lem:BoundOfE2j} bounds the probability in the above expression. Lemma \ref{lem:sum1} then bounds the summation, giving the desired bound in this case.
	
	When \(\tilde{\eta} = 0 \), \( N_{B,a}\le T \), \( B_{T^k_a} = 1 \) and \( T^k_a \ge K+k-1 \). The required average is bounded by 
	\[ \sum\limits_{k=2}^{T} e^{-k c_\mu\delta^2} \le \frac{1}{c_\mu\delta^2}  + 1,\]
	giving the desired bound. 
\end{proof}

\begin{lemma}\label{lem:sum1}
	For \(\delta > 0\), there exists a constant \(\tilde{c}_\mu(\delta)\) (independent of \(T\)) such that
	\[ \sum\limits_{k=1}^{\frac{\log T}{\log (1+\tilde{\eta})} + 1} (1+\tilde{\eta})^{k-1} e^{-N_a(T^k_a) c_\mu\delta^2 } \le \tilde{c}_\mu(\delta) . \]
\end{lemma}

\begin{proof}
	Recall that \( (1+\tilde{\eta})^{k-2}\le  N_a(T^k_a)  \). The required summation is bounded by 
	\[ (1+\tilde{\eta})\sum\limits_{k=1}^{\frac{\log T}{\log(1+\tilde{\eta})} + 1} (1+\tilde{\eta})^{k-2} e^{-(1+\tilde{\eta})^{k-2} c_\mu\delta^2 },\]
	which is further bounded by
	\[(1+\tilde{\eta}) \int\limits_{1}^{\infty} (1+\tilde{\eta})^{k-2} e^{-(1+\tilde{\eta})^{k-2} c_\mu\delta^2} dk + \frac{1+\tilde{\eta}}{c_\mu\delta^2e},\]
	where, \( \frac{1}{c_\mu\delta^2e} \) is the maximum value of the function being summed. To see the above bound, we upper bound the required summation by sum of the integral of the function from \(1\) to \( \infty \) and the maximum value of the function. The above can then be shown to equal
	\[ \tilde{c}_\mu(\delta) :=  \frac{1+\tilde{\eta}}{c_\mu\delta^2}\lrp{\frac{e^{-c_\mu\delta^2/(1+\tilde{\eta})}}{\log(1+\tilde{\eta}) } +\frac{1}{e} }  , \]
	which is the required upper bound. 
\end{proof}

\subsection{Bounding the deviations of optimal arm}\label{app:OptimalArmDeviations}
Recall that, for \(k\ge 2,\) \( T^{k}_{a} \) denotes the random time of the beginning of the batch when the \(a^{th}\) arm won for the \(k^{th}\) time. In particular, arm \(a\) has been sampled for \(k-1\) batches till this time. Thus, \( T^k_a\) is at least \(K-1 + 1+\tilde{\eta}+\dots+\tilde{\eta}(1+\tilde{\eta})^{k-3} =K-1+(1+\tilde{\eta})^{k-2} \). Using this bound on \(T^k_a\), the following lemma follows directly from Proposition \ref{prop:deviation_self_normalizing}.
\begin{lemma}\label{lem:deviation_of_best_arm}
	For \(k\ge 2\), \(g_1(t) = \log(t)+2\log\log(t)+2\log(1+N_1(t))+1 \), 
	\[\mathbb{P}\lrp{{N_1(T^{k}_{a})}\KLinf\lrp{\hat{\mu}_1(T^{k}_{a}), m(\mu_1)} \geq {g_1(T^{k}_{a})}} \leq {(1+\tilde{\eta})^{-k+2}}{\lrp{\log\lrp{K-1+(1+\tilde{\eta})^{k-2}}}^{-2} }.\]
\end{lemma}

Now, recall that 
$$D_N := \sum\limits_{j=K+1}^{N}B_j \mathbbm{1}\lrp{U_1(T_j) \leq m, ~ A_{T_j} = a}$$

\begin{lemma}\label{lem:BoundingDN}
	For \( T > K \), 
	\[ \Exp{D_N}  \le 
	\begin{cases}
	&\lrp{1+\tilde{\eta}}\lrp{\frac{1}{(\log K)^2} + \frac{\pi^2}{6(\log(1+\tilde{\eta}))^2}}, \text{ for } \tilde{\eta} > 0\\
	&\frac{1+\log(K+1)}{(\log(K+1))^2},\text{ for } \tilde{\eta} =0.
	\end{cases} 
	\]
	
\end{lemma}
\begin{proof}
	Recall that \(N_{B,a}\) denotes the number of batches allocated to arm \(a\) in time T and $T^k_{a}$ denotes the time of beginning of batch when arm \(a\) won for the \(k^{th}\) time. Then \(D_N\) can be re-written as
	\[ \sum\limits_{k=2}^{N_{B,a}(T)} B_{T^k_a} \mathbbm{1}\lrp{U_1(T^k_a)\le m}.  \]
	Recall that for any \(t\), the event \(\lrset{U_1(t)\le m}\) is same as \(\lrset{N_1(t)\KLinf(\hat{\mu}_1(t),m ) \ge g_1(t)} \), giving 
	\[ D_N = \sum\limits_{k=2}^{N_{B,a}} B_{T^k_a} \mathbbm{1}\lrp{ N_1(T^k_a) \KLinf\lrp{\hat{\mu}_1(T^{k}_{a}), {m}}\geq g_1(T^{k}_{a}}. \]
	
	Let us first consider the case when \( \tilde{\eta} > 0 \). In this case, \( N_{B,a}  \) is at most \( \frac{\log(T)}{\log(1+\tilde{\eta})} \) and \( B_{T^k_a} \) is at most \(\tilde{\eta} N_a(T^k_a)+1 \), which in turn is at most \((1+\tilde{\eta})^{k-1}\). Thus, the required expectation is at most 
	\begin{align*}
	\Exp{D_N}\le \sum\limits_{k=2}^{\frac{\log(T)}{\log(1+\tilde{\eta})}+1} \lrp{1+\tilde{\eta}}^{k-1} \mathbb{P}\lrp{N_1(T^{k}_{a})\KLinf\lrp{\hat{\mu}_1(T^{k}_{a}), {m}}\geq g_1(T^{k}_{a})}. 
	\end{align*}
	
	For \(g_1(t) = \log(t)+2\log\log(t)+2\log(1+N_a(t))+1\), Lemma \ref{lem:deviation_of_best_arm} bounds the probability in the expression in the r.h.s. above. Summing over \(k \in \lrset{2,\dots \log(T)/\log(1+\tilde{\eta})} \), we get  
	\[ \Exp{D_N}  \le \lrp{1+\tilde{\eta}}\lrp{\frac{1}{(\log K)^2} + \frac{\pi^2}{6(\log(1+\tilde{\eta}))^2}}.\]
	
	Now, let \( \tilde{\eta} = 0 \). In this case, \( N_{B,a}\le T \), \( B_{T^k_{a}} = 1 \), \( T^k_a \ge K+k-1 \). Using these, together with \( g_1(T^k_a) \ge \log(K+k-1) + 2\log\log(K+k-1) + 2\log(1+N_1(T^k_a)) +1 \), we get 
	\[ \Exp{D_N} \le \sum\limits_{k=2}^T \mathbb{P}\lrp{ N_1(T^k_a)\KLinf(\hat{\mu}_a(T^k_a),m) \ge g_1(T^k_a) } \le \sum\limits_{k=2}^T (K+k-1)\inv \lrp{\log(K+k-1)}^{-2},\]
	which is bounded by the constant in the statement. 
\end{proof}

\subsection{Bounding the number of pulls of a sub-optimal arm}\label{proof:lem:BoundingSumE1j} 
Recall that for \(T\ge K+1\), \(N\) denotes the random number of batches played by the algorithm till time \(T\), \(B_j \) denotes the random number of samples allocated within the \(j^{th}\) batch, and \(T_j\) denotes the time of beginning of the \(j^{th}\) batch. Furthermore, recall 
\[E_{1j} = \mathbbm{1}\lrp{\KLinf\lrp{\hat{\mu}_a(T_j), m} \leq \frac{g_a(T_j)}{N_a(T_j)},~\KLinf\lrp{\hat{\mu}_a(T_j), m} > \KLinf(\mu_a,m)-\delta,~ A_{T_j} = a }.\]

This corresponds to the event when sufficient samples have not been allocated to the sub-optimal arm \(a\), and contributes to the regret of the algorithm. Clearly, 
\begin{align*}
	\sum\limits_{j=1}^{N}B_jE_{1j} &\le \sum\limits_{j=1}^{N}B_j\mathbbm{1}\lrp{N_a(T_j)\leq \frac{g_a(T_j)}{\KLinf\lrp{\mu_a,m}-\delta}, A_{T_j}=a} .
\end{align*}
	
\begin{lemma}\label{lem:BoundingSumE1j}For \( T\ge K+1 \), \( \tilde{\eta} \ge 0 \), \( \delta>0 \),
	\[
	\sum\limits_{j=1}^{N}B_jE_{1j}\le (1+\tilde{\eta})\lrp{\frac{\log(T)}{\KLinf(\mu_a,m)-\delta}+ O\lrp{\log\log(T)}}. 
	\]
\end{lemma}
\begin{proof}
Since \(\log(t)+2\log\log(t)\) is a monotonically increasing function, using the form of \( g_a(.) \),  
\begin{align*}
	\sum\limits_{j=1}^{N}B_jE_{1j} &\leq \sum\limits_{j=1}^{N}B_j\mathbbm{1}\lrp{N_a(T_j)\leq \frac{\log(T) + 2\log\log(T) + 2\log(1+N_a(T_j))+1}{\KLinf\lrp{\mu_a,m}-\delta}, ~A_{T_j}=a }. 
\end{align*}
	
Clearly, \( z^* := \sup\lrset{z\in\mathbb{N} :  z d \le {\log(T) + 2\log\log(T) + 2\log(1+z) + 1}}\) is at most $\tilde{z}$, which equals 
\[\frac{\log(T)+2\log\log(T)}{d}+ \lrp{1+\frac{2}{d}}\log\lrp{1+\frac{\log(T)+2\log\log(T)}{d}} + \frac{10}{d} + O\lrp{\log\log\log(T)}.\numberthis\label{eq:exact_bound}\]
	
Thus, setting \( d = \KLinf(\mu_a,m)-\delta \), we get that 
\begin{align*}
	\sum\limits_{j=1}^{N}B_jE_{1j} &\leq B_N\mathbbm{1}\lrp{N_a(T_N)\leq \frac{\log(T) + 2\log\log(T) + 2\log(1+N_a(T_N))+1}{\KLinf\lrp{\mu_a,m}-\delta}, ~A_{T_j}=a } + \tilde{z}. 
\end{align*}

Clearly, when the indicator above is \(1\), then \(  B_N \) is at most \( \lrp{\tilde{\eta}\tilde{z}+1} \). Thus, 
\[
	\sum\limits_{j=1}^{N}B_jE_{1j}\le (1+\tilde{\eta})\lrp{\frac{\log(T)}{\KLinf(\mu_a,m)-\delta}+ O\lrp{\log\log(T)}}, 
\]
where the lower order terms in the above expression are the \(  o(\log(T))\) terms in (\ref{eq:exact_bound}), with \( d= \KLinf(\mu_a,m)-\delta \). 
\end{proof}

\subsection{Towards proving Proposition \ref{prop:deviation_self_normalizing}}\label{sec:proof:prop:deviation_self_normalizing}
In this section, we prove the anytime concentration inequality in Proposition \ref{prop:deviation_self_normalizing}. The proof involves constructing mixtures of super-martingales using dual formulation for \( \KLinf \), and may also be of independent interest. Lemma \ref{lem:exp-concave} below is borrowed from \cite{agrawal2020optimal}, and is stated here for completeness.  
\begin{lemma}\label{lem:exp-concave}
		Let $\Lambda \subseteq \mathbb R^d$ be a compact and convex subset and $q$ be the uniform distribution on $\Lambda$. Let $g_t: \Lambda \to \mathbb R$ be any series of exp-concave functions. Then
		\[
		\max_{\bm{\lambda} \in \Lambda}~
		\sum_{t=1}^T g_t(\bm{\lambda})
		~\le~
		\log \ex_{\bm{\lambda}\sim q}\lrp{
			e^{\sum_{t=1}^T g_t(\bm{\lambda})}
		}
		+ d\log(T+1)+1.
		\]
\end{lemma}

\subsubsection{Proof of Proposition \ref{prop:deviation_self_normalizing}}
We first note that if \( \mu_b = \delta_{B^{1/(1+\epsilon)}} \), which is the only distribution in \( \mathcal L_B \)  with the maximum possible mean in this class, i.e., \( B^{1/(1+\epsilon)} \), then the statement is vacuously true, since \( \hat{\mu}_b(n) = \mu_b \). Whence \( \KLinf(\hat{\mu}_b,\mu_b) =0 \). 

We now consider \( \mu_b \ne \delta_{B^{1/(1+\epsilon)}} \). For \(\epsilon > 0 \) and \( B > 0\), let \( f(y) := \abs{y}^{1+\epsilon} \), and for \(c \ge 0\), define \(f\inv(c) = c^{1/(1+\epsilon)}\). From the dual formulation in Lemma \ref{lem:DualFormulation}, \[ N_b(n)\KLinf(\hat{\mu_b}(n),m(\mu_b)) = \max\limits_{\bm{\lambda}\in\mathcal R(m(\mu_b),B)}~ \E{\mu_b}{\log\lrp{1-\lambda_1(X-m(\mu_b))-\lambda_2(B-f(X))}}, \]
where \(\mathcal{R}(m(\mu_b),B) \subset \Re^2 \). Let \(q\) be uniform distribution on \(\mathcal{R}(m(\mu_b),B)\), which is defined since the region is a  compact set (see Lemma \ref{rem:compactness_of_R}). Define 
\[ U_b(n) = \E{\bm{\lambda}\sim q}{\prod\limits_{i=1}^{N_b(n)} \lrp{1-\lambda_1(X_i-m(\mu_b))-\lambda_2(B-f(X_i))} \big\lvert X_1,\dots,X_{N_b(n)} } \text{ for } n\ge 1,\]
where \( \lrset{X_1,\dots,X_{N_b(n)}} \) are \( N_b(n) \) samples generated from arm \(b\) in time \(n\). 

Setting \(d=2\), \( g_i(X) = \log(1-\lambda_1(X_i-m(\mu_b))-\lambda_2(B-f(X_i))) \), \(\) in Lemma \ref{lem:exp-concave}, on each sample path, we have 
\[ N_b(n)\KLinf(\hat{\mu}_b(n),m(\mu_b)) \le \log U_b(n) + 2\log(1+N_b(n)) + 1. \numberthis\label{eq:bound_kl}\] 

Since \(\mu_b \in \cal L\), \( U_b(n) \) is a non-negative super-martingale with mean at most \(1  \). Using this in (\ref{eq:bound_kl}), together with Ville's inequality, we get 

\[ \mathbb{P}\lrp{\exists n\in \mathbb{N}: ~ N_b(n)\KLinf(\hat{\mu}_b(n), m(\mu_b))-2\log(1+N_b(n)) -1 \ge x} \le e^{-x},  \]
proving the desired inequality. $\Box$\\

\section{Computing the index}\label{appx:index.compute}
Let $\mathcal L \subseteq \mathcal{P}(\Re)$ be the set of distributions on $\Re$ with $(1+\epsilon)^\text{th}$ moment bounded by $B$. For \(\eta \in \mathcal{P}(\Re)\) itself possibly outside $\mathcal L$, we define the upper index at threshold $C$ by (\ref{eq:index_reformulated}), which we re-state
\[
  U_\eta
  ~:=~
  \max_\kappa~ m(\kappa)
  \qquad
  \text{s.t.}
  \qquad
  \kappa \in \mathcal L
  \quad \text{and} \quad
  \KL(\eta, \kappa) \le C.
\]
We now give a characterization of the above optimisation problem. Note that the variation with constraint $N \KL(\eta,\kappa) \le C$ follows from the below result after dividing $C$ by $N$.

\begin{lemma}[Dual formulation of index \( U_\eta\)]\label{lem:DualFormulationIndex}
  \[
    U_\eta
    ~=~
  \min_{\lambda_1, \lambda_2}~
  \lambda_1
  + \lambda_2 B
  - e^{\int \eta(x) \ln \del*{
      \lambda_1
      - x
      + \lambda_2 \abs{x}^{1+\epsilon}} \dif x
    -  C
  }
  \quad
  \text{s.t.}
  \quad
    \lambda_1
    \ge
    \lambda_2^{-\frac{1}{\epsilon }} \frac{
      \epsilon
    }{
      (1+\epsilon)^{1+\frac{1}{\epsilon }}
    }
    \text{ and }
    \lambda_2 \ge 0
    .
\]
\end{lemma}

\begin{proof}
  The proof is a diligent application of convex duality. Introducing Lagrange multipliers $\lambda_1, \lambda_2$ and $\lambda_3$ for the normalisation, moment and KL-ball constraints respectively, we find that the objective equals
\[
  \min_{\lambda_1, \lambda_2 \ge 0, \lambda_3 \ge 0}
  \max_{\kappa \ge 0}~
  \ex_\kappa [X]
  + \lambda_1 \del*{1 - \ex_\kappa \sbr*{1}}
  + \lambda_2 \del*{B - \ex_\kappa \sbr*{\abs{X}^{1+\epsilon}}}
  + \lambda_3 \del*{C - \KL(\eta, \kappa)}
\]
The solution for $\kappa$ is
\[
  \kappa(x)
  ~=~
  \frac{
    \lambda_3 \eta(x)
  }{
    \lambda_1
    - x
    + \lambda_2 \abs{x}^{1+\epsilon}}
\]
where we inherit the restriction $(\lambda_1,\lambda_2) \in \mathcal R \df \setc*{(\lambda_1, \lambda_2) \in \mathbb R^2}{\forall x \in \mathbb R : \lambda_1 - x + \lambda_2 \abs{x}^{1+\epsilon} \ge 0}$ (which in particular implies that $\lambda_1 \ge 0$, and $\kappa$ has at most one additional support point compared to $\eta$, which must then be at
$(\lambda_2 (1+\epsilon))^{-1/\epsilon }$). Plugging this in, we find
\[
  \min_{(\lambda_1, \lambda_2) \in \mathcal R, \lambda_3 \ge 0}~
  - \lambda_3
  + \lambda_3 \int \eta(x) \ln \del*{\frac{
    \lambda_3
  }{
    \lambda_1
    - x
    + \lambda_2 \abs{x}^{1+\epsilon}}} \dif x
  + \lambda_1
  + \lambda_2 B
  + \lambda_3 C
\]
Optimising for $\lambda_3$ gives
\[
  \lambda_3
  ~=~
  e^{\int \eta(x) \ln \del*{
    \lambda_1
    - x
    + \lambda_2 \abs{x}^{1+\epsilon}} \dif x
  -  C
}
\]
and plugging this in gives the claim.
\end{proof}

The upshot of this result is that our $\KLinf$-based indices can be computed with convex optimisation tools. The good news is that the optimisation variable has dimension $2$, making standard convex optimisation including e.g.\ the ellipsoid method practical. Note thought at the optimisation region for $(\lambda_1,\lambda_2)$ is unbounded, which may be addressed by successively enlargign the starting ellipsoid. When applying this to an empirical distribution $\eta$ supported on $n$ points, the number of terms in the objective (and hence the run time) scales linearly with $n$ and also with the number of bits of precision required.

\section{Comparison with Robust-UCB of \cite{bubeck2013heavytail}}\label{App:ComparisonWithBubeck}
Consider the following optimization problem that corresponds to our index: 
\[ \max\limits_{\kappa\in\cal L}~~ \E{\kappa}{X} \quad \text{s.t.} \quad n\KL(\hat{\eta}(n), \kappa) \le C.\]
The optimal value above is at most 
\[ \max\limits_{\kappa\in\cal L} ~~ \E{\kappa}{X}\quad \text{s.t.} \quad n\E{\hat{\eta}(n)}{g(X)}  - n\log \E{\kappa}{e^{g(X)}} \le C,\]
where we used the Donsker-Varadhan dual representation for \( \KL \). For a sequence of thresholds \(u_n = \lrp{B n (\log\delta\inv)\inv}^{{1}/\lrp{1+\epsilon}}\), and $\theta > 0$, define a function \(g_n(X) = X \mathbbm{1}\lrp{\abs{X}\le u_n}.\)
Substituting \(g_\eta\) for \(g\) in the above, and adding \(n\theta\E{\kappa}{X}\) on both the sides, we get the following upper bound on our index: 
\[ \max\limits_{\kappa\in\cal L} ~\E{\kappa}{X} ~~ \text{s.t. } \theta \sum\limits_{i=1}^n \lrp{\E{\kappa}{X}-X_i\mathbbm{1}\lrp{\abs{X_i}\le u_n}} -n\log\E{\kappa}{e^{g_n(X)}} \le C + n \theta\E{\kappa}{X}.\numberthis\label{eq:indexbound} \]
  
Let \( Y_n =  X\mathbbm{1}\lrp{\abs{X}\le u_n} \) and \( m_n = \E{\kappa}{X\mathbbm{1}\lrp{\abs{X}\le u_n}} \). Then, $\E{\kappa}{\theta^2 Y^2_n} \le \theta^2 Bu^{1-\epsilon}_n $. Using this, 
\begin{align*}
	\E{\kappa}{e^{-\theta X\mathbbm{1}\lrp{\abs{X}\le u_n}}} &\le 1 -\theta m_n + \sum\limits_{j=2}^\infty \frac{\E{\kappa}{\abs{\theta Y_n}^j}}{j!} \le 1-\theta m_i + \frac{B}{u^{1+\epsilon}_n}\sum\limits_{j=2}^{\infty} \frac{(\theta u_n)^j}{j!}. \numberthis\label{eq:UniformMGF} 
\end{align*}

Thus, we have \(\E{\kappa}{e^{-\theta X\mathbbm{1}\lrp{\abs{X}\le u_n}}}  \le 1-\theta m_n + \frac{B}{u^{1+\epsilon}_n} \lrp{e^{\theta u_n } - \theta u_n -1}  \). Using $1+x\le e^x$ and (\ref{eq:UniformMGF}) in (\ref{eq:indexbound}), we get that the optimal value of the following optimization problem is an upper bound on our index:
\[ 
\max\limits_{\kappa\in\cal L}~~ \E{\kappa}{X} \]
subject to 
\[ \theta \sum\limits_{i=1}^n \lrp{\E{\kappa}{X}-X_i\mathbbm{1}\lrp{\abs{X_i}\le u_n}}  \le C + n \lrp{\theta\E{\kappa}{X} -\theta m_n + \frac{B}{u^{1+\epsilon}_n} \lrp{e^{\theta u_n}-\theta u_n -1} }. 
\]
Clearly, \( \E{\kappa}{X\mathbbm{1}\lrp{\abs{X} \ge u_n}} \) is at most \( B/(u_n)^{\epsilon} \). The constraint can be relaxed  to 
\[   \frac{1}{n}\sum\limits_{i=1}^n \lrp{\E{\kappa}{X}-X_i\mathbbm{1}\lrp{\abs{X_i}\le u_n}}  \le \frac{B}{u^{\epsilon}_n} + \frac{1}{\theta} \lrp{\frac{C}{n}+ \frac{B}{u^{1+\epsilon}_n} \lrp{e^{\theta u_n}-\theta u_n -1}}  \]

Choosing \( \theta = \frac{Cu^{\epsilon}_n}{nB} \), the above constraint is \( \frac{1}{n}\sum\limits_{i=1}^n \lrp{\E{\kappa}{X}-X_i\mathbbm{1}\lrp{\abs{X_i} u_n}}\) less than 

\[  \sum\limits_{i=1}^n \E{\kappa}{X\mathbbm{1}\lrp{\abs{X} \ge u_n }} + {\lrp{nB}^{\frac{1}{1+\epsilon}}}\frac{\log\delta\inv}{C} \lrp{\log\delta\inv}^{\frac{\epsilon}{1+\epsilon}}\lrp{e^{\frac{C}{\log \delta\inv}} - 1}. \]
Setting \(\hat{\mu}_T(n) = \frac{1}{n} \sum\limits_{i=1}^n X_i\mathbbm{1}\lrp{\abs{X_i}\le u_n} \), and using that  \(u_n = \lrp{B n \lrp{\log\delta\inv}\inv}^{\frac{1}{1+\epsilon}}\) we get the following bound on the index:
\[ \hat{\mu}_T(n) + {B}^{\frac{1}{1+\epsilon}} \lrp{\frac{\log\delta\inv}{n}}^{\frac{\epsilon}{1+\epsilon}}\lrp{1+ \lrp{e^{\frac{C}{\log\delta\inv}}-1}\frac{\log\delta\inv}{C}} .\]

\section{Finite time bound for bounded support distributions: Proof of Proposition \ref{prop:BddSupp} }\label{app:sec:finiteBoundForBdd}
In this section, we establish the conjectured optimality of the empirical KL-UCB algorithm of \cite{cappe2013kullback} and give the first optimal finite-time regret bound for bounded-support arm distributions. In this setting, \(\mathcal L = \mathcal P([0,1])\), and for \(\eta\in\mathcal P(\Re)\), \(\KLinfR{L}(\eta,x) \) is defined to be \(\inf\KL(\eta,\kappa) : ~ \kappa \in \mathcal L, \text{ and } m(\kappa) \ge x\).  \cite{HondaBounded10} develop alternate representations for the \( \KLinfR{\cal L} \) in this setting. They show that 
\[ \KLinfR{\mathcal L }(\mu_a,x) := \max\limits_{\lambda \in \left[0,\frac{1}{1-x}\right]} \E{\mu_a}{\log\lrp{1-(X-x)\lambda}} \quad \text{ for } x\in[0,1]. \quad\]
\(\KLinf\)-UCB, with \(\mathcal L = \mathcal P([0,1])\), \(\KLinfR{\mathcal L}\) defined above, and \(g(t) = \log(t)+2\log(\log(t))\),  recovers the empirical KL-UCB algorithm of \cite{cappe2013kullback}. In particular, the index for arm \(a\), denoted as \(U_a(t)\), is given by 
\[ U_a(t) = \max\lrset{x : \KLinfR{\mathcal L}(\hat{\mu}_a(t),x) \le \frac{g(t)}{N_a(t)} }.\]
We highlight the key steps in the proof of the finite time bound, below. 

As earlier, we bound the average number of pulls of a sub-optimal arm, \(a\). For simplicity of notation, let us assume that arm 1 is the arm with maximum mean. Let \(\epsilon_1\in (0,m(\mu_1)) \), and \( \tilde{m} = m(\mu_1) - \epsilon_1 \). Then for arm \(a\ne 1\), using the definition of the index, we have 
\begin{align*} 
{N_a(T)} &= 1+\sum\limits_{t=K+1}^{T} {\mathbbm{1}\lrp{A_t=a}} = D_T + E_T + 1,  
\end{align*}
where \(A_t  \) denotes the arm selected at time \(t\), and  
\[ D_T := \sum\limits_{t=K+1}^T \mathbbm{1}\lrp{\KLinfR{\mathcal L}\lrp{\hat{\mu}_1(t),\tilde{m}} \ge \frac{g(t)}{N_1(t)} ~~ \text{and}~~ A_t = a},  \]
and 
\[E_T := \sum\limits_{t=K+1}^T \mathbbm{1}\lrp{\KLinfR{\mathcal L}\lrp{\hat{\mu}_a(t),\tilde{m}} < \frac{g(t)}{N_a(t)} ~~\text{and}~~ A_t= a }. \]

Using \citet[Section B.2, (26)]{cappe2013kullbackSupp} and the bounds therein, and that \( \epsilon_1 < 1 \),
\begin{align*}
\Exp{D_T} \le  \lrp{\sum\limits_{t=K+1}^T e^{-g(t)}} \lrp{3e+2+\frac{4}{\epsilon_1^2} + \frac{8e}{\epsilon_1^4} }\le \frac{36}{\epsilon_1^4}\lrp{\sum\limits_{t=K+1}^T e^{-g(t)}}.
\end{align*}
For \(t \ge 2 \) and \(K\ge 2\) we have 
\[ \Exp{D_T} \le \frac{36}{\epsilon_1^4} \sum\limits_{t=K+1}^T \frac{1}{t\log t} \le \frac{36}{\epsilon_1^4}\lrp{\frac{1}{2\log 2} + \int\limits_{2}^{T} \frac{1}{t \log t} ~dt }\le \frac{36}{\epsilon_1^4} \lrp{2+ \log \log T}.\numberthis\label{eq:boundOnD_TBdd}  \]
Moreover, $\Exp{E_T}$ is bounded by 
\begin{align*}
&\sum\limits_{t=K+1}^T\mathbb{P}\lrp{ \KLinfR{\mathcal L}(\hat{\mu}_a(t),\tilde{m}) \le \frac{g(t)}{N_a(t)}; ~~\KLinfR{\mathcal L}( \hat{\mu}_a(t),\tilde{m} ) \ge \KLinfR{\mathcal L}(\mu_a,\tilde{m}) - \delta; ~~ A_t=a }\\&+ \sum\limits_{t=K+1}^T \mathbb{P}\lrp{ A_t = a; ~~\KLinfR{\mathcal L}( \hat{\mu}_a(t),\tilde{m} ) \le \KLinfR{\mathcal L}(\mu_a,\tilde{m}) - \delta },
\end{align*}
which is further bounded by 
\begin{align*}
&\sum\limits_{t=K+1}^T\mathbb{P}\lrp{  \KLinfR{\mathcal L}(\mu_a,\tilde{m}) - \delta \le \frac{g(t)}{N_a(t)}; ~~ A_t=a } + \sum\limits_{s=1}^T \mathbb{P}\lrp{\KLinfR{\mathcal L}( \hat{\mu}_{a,s},\tilde{m} ) \le \KLinfR{\mathcal L}(\mu_a,\tilde{m}) - \delta },
\end{align*}
where \( \hat{\mu}_{a,s} \) denotes the empirical distribution for arm \(a\), corresponding to \(s\) samples from that arm. Clearly, the first term in the above summation is at most 
\[ {g(T)}\lrp{\KLinfR{\mathcal L}(\mu_a,\tilde{m})-\delta}\inv + 1, \]
while the second term is bounded using \citet[Theorem 12]{honda2015SemiBounded} by
\[\sum\limits_{s=1}^{T} e^{-s c(\delta,\epsilon_1)} \le \lrp{1-e^{-c(\delta,\epsilon_1)}}\inv, \quad\text{where} \quad c(\delta,\epsilon_1) = \frac{\delta^2}{2\lrp{c_0+\frac{1-m(\mu_a)}{1-\tilde{m}}}}, \numberthis\label{eq:cdeltaeps_Bdd} \]
with the condition that \( \delta \le \frac{1}{2} \lrp{c_0 + \frac{1-m(\mu_a)}{1-\tilde{m}}} \), and \(c_0 \ge 2.2\). Thus, we have 
\begin{align*}
\Exp{E_T} \le {g(T)}\lrp{\KLinfR{\mathcal L}(\mu_a, \tilde{m})-\delta}\inv + 1 + \frac{1}{1-e^{-c(\delta,\epsilon_1)}}, 
\end{align*}
where \( c(\delta,\epsilon_1) \) is specified above. Using \citet[Lemma 4]{cappe2013kullbackSupp},
\begin{align*}
\Exp{E_T} \le {g(T)}\lrp{\KLinfR{\mathcal L}(\mu_a, m(\mu_1)) -\frac{\epsilon_1}{1-m(\mu_1)} -\delta}\inv + 1 + \lrp{1-e^{-c(\delta,\epsilon_1)}}\inv, \numberthis\label{eq:boundOnET_Bdd}
\end{align*}

Using (\ref{eq:boundOnD_TBdd}) and (\ref{eq:boundOnET_Bdd}) above, average number of pulls of a sub-optimal arm is bounded by 
\begin{align*}
{g(T)}\lrp{\KLinfR{\mathcal L}(\mu_a, m(\mu_1)) -\frac{\epsilon_1}{1-m(\mu_1)}-\delta}\inv + 1 + \lrp{1-e^{-c(\delta,\epsilon_1)}}\inv + \frac{36}{\epsilon_1^4}\lrp{2 + \log \log T},
\end{align*}
with \( c(\delta,\epsilon_1) \) specified in (\ref{eq:cdeltaeps_Bdd}). This bound can be optimized over \( \epsilon_1 \) and \( \delta \) under the conditions that \( \epsilon_1 < m(\mu_1) \) and  \( \delta \le \frac{1}{2} \lrp{c_0 + \frac{1-m(\mu_a)}{1-m(\mu_1)+\epsilon_1}}. \) Choosing 
\[ \delta^3 =8 \lrp{c_0 + \frac{1-m(\mu_a)}{1-m(\mu_1)+\epsilon_1}} \lrp{\frac{\lrp{\KLinfR{\mathcal L}(\mu_a,m(\mu_1))}^2}{g(T)}}\]
and \[\epsilon_1^5 = \frac{\lrp{\KLinfR{\mathcal L}(\mu_a,m(\mu_1))}^2}{g(T)}\lrp{2 + \log \log T}, \] 
the number of times a sub-optimal arm is pulled is bounded by
\[ \lrp{\log T + \log \log T}\lrp{\KLinfR{\mathcal L}(\mu_a,m(\mu_1)) - O\lrp{\lrp{\frac{\log \log T}{\log T}}}^{\frac{1}{5}}}\inv + O\lrp{\lrp{\log T}^{\frac{4}{5}} \lrp{\log\log(T)}^\frac{1}{5} }.  \]

\section{Proving Theorem \ref{th:AsymptoticOptimality2}}\label{app:perturbations}

In this section, we establish the theoretical guarantees for \( \KLinf \)-UCB2 algorithm, which is a perturbed version of \( \KLinf \)-UCB. Recall that for \( \epsilon_1 > 0 \), we define \( \tilde{B}=B+\epsilon_1 \), and let \( \delta'_t = \log(1+(\log\log(t))\inv) \). Given \(\mu\in\mathcal L^K  \), for \(\tilde{\eta} \ge 0\), \( \KLinf \)-UCB2 is precisely \(\KLinf \)-UCB(K,$\tilde{B}$,\(\epsilon\), \(\tilde{\eta},(1+\delta'_t)^2\log(t) \)). Proof of Theorem \ref{th:AsymptoticOptimality2} follows exactly along the lines of Theorem \ref{th:geometric}. We highlight only the differences here. 

As earlier, we analyse the events leading to selection of a sub-optimal arm by the algorithm. For \(\epsilon_2 > 0\), in this section, let \(m=m(\mu_1)-\epsilon_2\). The event that at the beginning of \( j^{th} \) batch, sub-optimal arm, \(a\), has the maximum index, i.e., \(\lrset{A_{T_j}=a}\) for \(a\ne 1\), equals 
\begin{align*}
\lrset{U_1(T_j) \leq m ~~\text{ and }~~ A_{T_j}=a }~ \bigcup ~\lrset{U_a(T_j) > m ~~\text{ and }~~ A_{T_j}=a }. \numberthis\label{eq:IncorrectSelectionEvent2}
\end{align*}

Let \(N\) denote the random number of batches till time \(T\). As earlier, the random variable of interest is 
\begin{equation*}
\Exp{N_a(T)} = 1 + \Exp{\sum\limits_{j=K+1}^{N} B_j \mathbbm{1}\lrp{A_{T_j}=a}} ~= ~ 1 + \Exp{D_N} + \Exp{E_N},
\end{equation*}
where, using the division from (\ref{eq:IncorrectSelectionEvent2}), we define
\[ D_N := \sum\limits_{j=K+1}^{N}B_j \mathbbm{1}\lrp{U_1(T_j) \leq m, ~ A_{T_j} = a}, ~\text{and }~ E_N := \sum\limits_{j=K+1}^{N}B_j \mathbbm{1}\lrp{U_a(T_j)> m, ~ A_{T_j}=a}. \]

Proof for controlling the deviations of the sub-optimal arm, i.e., \(\Exp{E_N}\) above, follows exactly as earlier, with \( \KLinf(\mu_a,m(\mu_1)) \) replaced by \( \KLinfR{\epsilon_1}(\mu_a,m(\mu_1)-\epsilon_2) \). Thus for any \(\delta>0\), we will have 
\[ \Exp{E_N} \le \frac{(1+\tilde{\eta}) (1+\delta'_t)^2 \log(t)}{\KLinfR{\epsilon_1}(\mu_a,m(\mu_1)-\epsilon_2)-\delta} + O(1),\]
where \( O(1)\) terms involve constants that are functions of \(\epsilon_1,\epsilon_2, \) and \(\delta\). 

Let \(T^k_a\) denote the time of beginning of the batch when arm \(a\) won for the \(k^{th}\) time. Let us now show that \( \Exp{D_N} = O(1)\).  

As earlier,  
\begin{align}\label{eq:DN}
\Exp{D_N}\le \sum\limits_{k=2}^{\frac{\log(T)}{\log(1+\tilde{\eta})}+1} \lrp{1+\tilde{\eta}}^{k-1} \mathbb{P}\lrp{N_1(T^{k}_{a})\KLinf\lrp{\hat{\mu}_1(T^{k}_{a}), {m}}\geq g_1(T^{k}_{a})}. 
\end{align}

Proposition \ref{th:DoublePeeling} bounds the probability in the summand above. Lemma \ref{lem:oLogTDn} argues that $ \Exp{D_N} = o(\log T). $

Thus, 
\[ \Exp{N_a(T)} \le \frac{(1+\tilde{\eta})(1+\delta'_t)^2 \log  T}{\KLinfR{\epsilon_1}(\mu_a,m(\mu_1)-\epsilon_2)-\delta} + o(\log T). \]

Dividing by \( \log(T) \) and taking limit as \(T \rightarrow \infty \),
\[ \limsup\limits_{T\rightarrow \infty} \frac{\Exp{N_a(T)}}{\log(T)} \le \frac{1+\tilde{\eta}}{\KLinfR{\epsilon_1}(\mu_a,m(\mu_1)-\epsilon_2)-\delta}. \] 

Since \( \epsilon_2>0 \) and \( \delta > 0 \) are arbitrary constants, taking infimum over these, and using that \( \KLinfR{\epsilon_1}(\mu_a, x) \) is a continuous function of \(x\), we get 
\[  \limsup\limits_{T\rightarrow \infty} \frac{\Exp{N_a(T)}}{\log(T)} \le \frac{1+\tilde{\eta}}{\KLinfR{\epsilon_1}(\mu_a,m(\mu_1))}.  \quad \quad \Box \]

For \(x\in\Re\), let \(f(x) = \abs{x}^{1+\epsilon}\), and for \(c \geq 0\) define \( f\inv(c) := c^{1/(1+\epsilon)}\). Let \(T > t > 0\), \( \gamma > 0 \), \(X \sim \mu_1\), 
\[ C_{1} := C_{2}\lrp{1-e^{\frac{-\epsilon_2}{f\inv(\tilde{B})-{m}}}}\inv \lrp{1-e^{\frac{-\epsilon_1}{\tilde{B}-f\lrp{m}}}}\inv,\] \[ C_T := \lrp{1-(1+\gamma)^{-1-\delta^{'}_T}} \text{ and }\; C_{2}:=\exp\lrp{\frac{\Exp{\abs{X-\tilde{m}}}}{f\inv(\tilde{B}) -m } + \frac{\Exp{\abs{\tilde{B}-f(X)}}}{\tilde{B}-f\lrp{m}} }.\]

\begin{proposition}\label{th:DoublePeeling}
	For \(T> K\), \(k>0\), \(\epsilon_1 > 0\), and \(g(t) = \lrp{1+\delta'_t}^2\log(t) \), 
	\[\mathbb{P}\lrp{\KLinfR{\epsilon_1}({\hat{\mu}_1(T^{k}_{a}), {m}}) \geq \frac{g(T^{k}_{a})}{N_1(T^{k}_{a})}} \leq \frac{C_{1}\lrp{1+\tilde{\eta}}^{-\lrp{k-1}({1+\delta^{'}_T})} }{\log\lrp{1+\delta'_T}}\lrp{\frac{\log{\lrp{1+\tilde{\eta}}^{k-1}} }{C_T}  + \frac{\log\lrp{1+\gamma}}{C^2_T} }.\]
\end{proposition}
In particular, introducing the perturbations allows us to get rid of the additional \( 2\log(1+N_1(t)) \) cost in the threshold, which resulted from the presence of this term in the bound in Proposition \ref{prop:deviation_self_normalizing}. Proposition \ref{th:DoublePeeling} is proved in Section \ref{app:proof_th_double_peeling} below. 

\begin{lemma}\label{lem:oLogTDn}
	For \(T>0\), \(\epsilon_1 > 0\) and \(g(t):= \lrp{1+\log\lrp{1+\frac{1}{\log\log t}}}^2\log(t), \)
	\[\sum\limits_{k=1}^{\frac{\log(T)}{\log(1+\tilde{\eta})}+1} \lrp{1+\tilde{\eta}}^{k} \mathbb{P}\lrp{N_1(T^{k}_{a})\KLinfR{\epsilon_1}\lrp{\hat{\mu}_1(T^{k}_{a}), {m}}\geq g(T^{k}_{a})} = o(\log\lrp{T}).\]
\end{lemma}

\begin{proof}
	Using Proposition \ref{th:DoublePeeling} to bound the probability in the summation, the required expression can be bounded by 
	\[ (1+\tilde{\eta})\sum\limits_{k=1}^{\frac{\log T}{\log \lrp{1+\tilde{\eta}}}+1} \lrp{\frac{C_1}{C_T}\frac{\log\lrp{1+\tilde{\eta}}}{\log\lrp{1+\delta'_T}} \frac{k - 1}{\lrp{1+\tilde{\eta}}^{\lrp{k-1}\delta^{'}_T}} +  \frac{C_1 \log\lrp{1+\gamma} }{C^2_T\log\lrp{1+\delta'_T}}\frac{1}{\lrp{1+\tilde{\eta}}^{\lrp{k-1}\delta^{'}_T} }},\]
	where \(C_1\), \(\delta^{'}_T\) and \(C_T\) are constants independent of \(k\), such that \(C_T \) converges to a constant and \(\delta^{'}_T\) converges to \(0\), as \(T \rightarrow \infty\) (see Proposition \ref{th:DoublePeeling}). It is then easy to see that 
	\[ \lrp{\log\lrp{1+\delta'_T} \lrp{1-\frac{1}{\lrp{1+\tilde{\eta}}^{\delta^{'}_T}}}^2}\inv = o(\log T),  \]
	where \(\delta^{'}_T = \log\lrp{1+\frac{1}{\log\log T}} \), and hence the required summation is \(o(\log T) \). 
\end{proof}

\subsection{Towards proving Proposition \ref{th:DoublePeeling}}\label{app:proof_th_double_peeling}
In this section, we prove the concentration inequality in Proposition \ref{th:DoublePeeling}. The proof involves use of peeling arguments, dual formulation for \(\KLinfR{\epsilon_1}\), \(\epsilon\)-nets and careful use of martingales, and may also be of independent interest. To facilitate the proof, we need some notation and results that will be used later, which we prove first.

In this section, for \(x\in\Re\), we define \(f(x) := \abs{x}^{1+\epsilon}\), and for \(c \geq 0\) \( f\inv(c) := c^{1/(1+\epsilon)}\). For \(T  > K\) and \(K < t\leq T\), define \(\delta'_t := \log\lrp{1+\frac{1}{\log \log t}}\).  For \(\gamma > 0\), let \(J_M = \frac{\log\lrp{T/\lrp{1+\tilde{\eta}}^{k-1}}}{\log\lrp{1+\gamma}}\). For \(j\in\lrset{0,1,\cdots, J_M}\), define the event 
\[D_{j} = \lrset{\lrp{1+\tilde{\eta}}^{k-1} \lrp{1+\gamma}^{j}\leq T^{k}_{a} \leq \lrp{1+\tilde{\eta}}^{k-1}\lrp{1+\gamma}^{j+1} },\]
and for \( \eta > 0\),  and \(i\geq 0\), define  
\[ C_{i}(t) = \lrset{\lrp{1+\eta}^{i}\leq N_1\lrp{t} \leq \lrp{1+\eta}^{i+1} } .\]

Furthermore, for \( \epsilon_2 > 0 \) recall that \( m = m(\mu_1)-\epsilon_2. \)

\begin{lemma} \label{lem:BoundOnSet} For \(\gamma>0, \eta > 0\), \(i\in\mathbb{N},\tilde{\eta}\geq 0\), \(k \geq 1 \), \( j\in [J_M] \), 
	
	\[ \mathbb{P}\lrp{N_1(T^{k}_{a}) \KLinfR{\epsilon_1}(\hat{\mu}_1(T^{k}_{a}) ,{m}) \geq g(T^{k}_{a}), D_j, C_i(T^{k}_{a})} \leq C_1 e^{ -\frac{g\lrp{\lrp{1+\tilde{\eta}}^{k-1}\lrp{1+\gamma}^j}}{\lrp{1+\eta}}}, \]
	where 
	\[C_1 = \frac{C_2}{1-e^{{-\epsilon_2}\lrp{f^{-1}(\tilde{B})-{m}}\inv}} \lrp{1-e^{\frac{-\epsilon_1}{\tilde{B}-f\lrp{m}}}}^{-1} ~\text{ and }~~ C_2=e^{\frac{\E{\mu_1}{\abs{X-m}}}{f^{-1}(\tilde{B}) -m }}e^{ \frac{\E{\mu_1}{\abs{\tilde{B}-f(X)}}}{\tilde{B}-f\lrp{m}} }.\]
\end{lemma}

\begin{proof}
	On the set \(C_i(T^{k}_{a})\) (denoted as \(C_i\) in this proof), \(N_1(T^{k}_{a}) \leq (1+\eta)^{i+1}\). Using this, the required probability is bounded from above by 
	\begin{equation}\label{eq:1}
	\mathbb{P}\lrp{\KLinfR{\epsilon_1}(\hat{\mu}_1(T^{k}_{a}) ,{m}) \geq \frac{g(T^{k}_{a})}{\lrp{1+\eta}^{i+1}} , D_j, C_i }.
	\end{equation}
	Using the dual formulation for \(\KLinfR{\epsilon_1}\) from Lemma \ref{lem:DualFormulation}, above is bounded by   
	\begin{align*}
	\mathbb{P}\lrp{\max\limits_{\bm{\lambda} \in \mathcal{R}({m},\tilde{B})}  \sum\limits_{l=1}^{N_1(T^{k}_{a})} \log\lrp{1-(X_l-{m})\lambda_1-(\tilde{B}-f\lrp{X_l}) \lambda_2} \geq \frac{N_1(T^{k}_{a})g(T^{k}_{a})}{\lrp{1+\eta}^{i+1}} , D_j, C_i},
	\end{align*}
	where \(\mathcal{R}({m},\tilde{B}) \) is a subset of \(\Re^2\) similar to that defined in the Lemma \ref{lem:DualFormulation}, such that the argument of \(\log\) in the expression above is always non-negative. Consider a \((\delta_1, \delta_2)\)-net over the rectangle 
	\[ \left[0, \frac{1}{f^{-1}(\tilde{B})-m} \right] \times \left[0,\frac{1}{\tilde{B}-f({m})}\right],  \]
	which contains the region \(\mathcal{R}({m},\tilde{B})\) (see Lemma \ref{rem:compactness_of_R}). Let \( \mathcal{G}_{l_1,l_2} \) denote a grid in the constructed net. We will choose the side lengths \(\delta_1 \) and \(\delta_2 \), later. Then using union bound over the grids in the net, probability in (\ref{eq:1}) can be bounded by 
	\begin{align*}
	\sum\limits_{l_1,l_2}\mathbb{P}\lrp{\max\limits_{\bm{\lambda} \in \mathcal{G}_{l_1,l_2}}  \sum\limits_{l=1}^{N_1(T^{k}_{a})} \log\lrp{1-(X_l-{m})\lambda_1-(\tilde{B}-f\lrp{X_l}) \lambda_2} \geq \frac{N_1(T^{k}_{a})g(T^{k}_{a})}{\lrp{1+\eta}^{i+1}} , D_j, C_i}.\numberthis \label{eq:bigSum}
	\end{align*}
	On the set \(D_j\), \(T^{k}_{a}\) is at least \( \underline{t}= \lrp{1+\tilde{\eta}}^{k-1}\lrp{1+\gamma}^{j}\). Thus, using monotonicity of \(g(\cdot)\), the probability in the summation above can be bounded by  
	\begin{equation}\label{eq:2}
	\mathbb{P}\lrp{\max\limits_{\bm{\lambda}\in \mathcal{G}_{l_1l_2}}  \sum\limits_{l=1}^{N_1(T^{k}_{a})} \log\lrp{1-(X_l-{m})\lambda_1-(\tilde{B}-f\lrp{X_l}) \lambda_2} \geq \frac{N_1(T^{k}_{a})g(\underline{t})}{\lrp{1+\eta}^{i+1}} , D_j, C_i}.
	\end{equation}
	Let the maximum in the expression above be attained at some point in the grid, say \((\lambda^*_1, \lambda^*_2)\). Furthermore, let \((\tilde{\lambda}_1, \tilde{\lambda}_2) \) denote one of the corner points of the grid \(\mathcal{G}_{l_1,l_2} \) such that \((\tilde{\lambda}_1, \tilde{\lambda}_2) \) is in the interior of \(\mathcal{R}({m},\tilde{B})\). Then, \({1-(X_l-{m})\lambda^*_1 - (\tilde{B}-f\lrp{X_l})\lambda^*_2} \) equals 
	\begin{align*}
	{1-(X_l-{m})\tilde{\lambda}_1 - (\tilde{B}-f\lrp{X_l})\tilde{\lambda}_2} +\lrp{X_l-{m}}\lrp{\tilde{\lambda}_1-\lambda^*_1} +\lrp{\tilde{B}-f\lrp{X_l}}\lrp{\tilde{\lambda}_2-\lambda^*_2} .
	\end{align*}
	Since \(\lambda^*_1\) and  \(\tilde{\lambda}_1 \) are in the same grid, they differ by at most \(\delta_1 \). Similarly, \(\tilde{\lambda}_2 \) and \(\lambda^*_2 \) differ by at most \(\delta_2 \). Thus, the r.h.s. in the above expression can be upper bounded by 
	\[{1-(X_l-{m})\tilde{\lambda}_1 - (\tilde{B}-f\lrp{X_l})\tilde{\lambda}_2 + \abs{X_l-{m}}\delta_1 + \abs{\tilde{B}-f\lrp{X_l}}\delta_2}. \]
	Let \(Y_{l} := \log\lrp{1-(X_l-{m})\tilde{\lambda}_1 - (\tilde{B}-f\lrp{X_l})\tilde{\lambda}_2 + \abs{X_l-{m}}\delta_1 + \abs{\tilde{B}-f\lrp{X_l}}\delta_2} \). Clearly, \(Y_{l}\) are i.i.d. random variables. Let \(Y\) be independent and identically distributed as \(Y_{l}\). The probability in (\ref{eq:2}) can then be bounded by
	
	\begin{equation}\label{eq:3}
	\mathbb{P}\lrp{ \sum\limits_{l=1}^{N_1(T^{k}_{a})} Y_{l} \geq \frac{N_1(T^{k}_{a})g\lrp{\lrp{1+\tilde{\eta}}^{k-1}\lrp{1+\gamma}^j}}{\lrp{1+\eta}^{i+1}} , D_j, C_i}.
	\end{equation}
	For \(0\leq \theta \leq 1\), let 
	\[\Lambda_{Y}(\theta) = \log{\Exp{e^{\theta Y}}} ~~\text{ and }~~ \theta^* = \argmax\limits_{0\leq\theta\leq 1} \lrset{ \frac{\theta g\lrp{\lrp{1+\tilde{\eta}}^{k-1}\lrp{1+\gamma}^j}}{\lrp{1+\eta}^{i+1} }-\Lambda_{Y}(\theta)} .  \]
	Clearly, \(\theta^* \geq 0\). Using Chernoff-like argument, we bound the expression in  (\ref{eq:3}) by 
	\[\mathbb{P}\lrp{e^{\sum\limits_{l=1}^{N_1(T^{k}_{a})} \lrp{\theta^* Y_{l}-\Lambda_{Y}(\theta^*)} } \geq e^{ N_1(T^{k}_{a}) \lrp{\frac{\theta^* g\lrp{\lrp{1+\tilde{\eta}}^{k-1}\lrp{1+\gamma}^j}}{\lrp{1+\eta}^{i+1}} - \Lambda_{Y}(\theta^*)}  }, D_j, C_i }. 
	\]
	Observe that by the choice of \(\theta^* \), the term in the exponent in the r.h.s. above is positive. Thus the above probability is bounded by choosing lower bound for \(N_1(T^{k}_{a})\) by 
	\[\mathbb{P}\lrp{e^{\sum\limits_{l=1}^{N_1(T^{k}_{a})} \lrp{\theta^* Y_{l}-\Lambda_{Y}(\theta^*)} } \geq e^{ (1+\eta)^{i} \lrp{\frac{\theta^* g\lrp{\lrp{1+\tilde{\eta}}^{k-1}\lrp{1+\gamma}^j}}{\lrp{1+\eta}^{i+1}} - \Lambda_{Y}(\theta^*)}  }, D_j, C_i },
	\]
	which can further be bounded by 
	\[\mathbb{P}\lrp{e^{ \mathbbm{1}\lrp{C_i\cap D_j} \sum\limits_{l=1}^{N_1(T^{k}_{a})} \lrp{\theta^* Y_{l}-\Lambda_{Y}(\theta^*)} } \geq e^{ (1+\eta)^{i} \lrp{\frac{\theta^* g\lrp{\lrp{1+\tilde{\eta}}^{k-1}\lrp{1+\gamma}^j}}{\lrp{1+\eta}^{i+1}} - \Lambda_{Y}(\theta^*)}  }}.\]
	From Markov's Inequality, and by the choice of \(\theta^*\), the above probability is less than 
	\begin{align*} 
	\Exp{e^{ \mathbbm{1}\lrp{C_i\cap D_j} \sum\limits_{l=1}^{N_1(T^{k}_{a})} \lrp{\theta^* Y_{l}-\Lambda_{Y}(\theta^*)} }} e^{ -(1+\eta)^{i} \max\limits_{0\leq\theta\leq 1} \lrp{\frac{\theta g\lrp{\lrp{1+\tilde{\eta}}^{k-1}\lrp{1+\gamma}^j}}{\lrp{1+\eta}^{i+1}} - \Lambda_{Y}(\theta)}  },
	\end{align*}
	which is less than 
	\[\Exp{e^{ \sum\limits_{l=1}^{N_1(T^{k}_{a})} \lrp{\theta^* Y_{l}-\Lambda_{Y}(\theta^*)} }} e^{ -(1+\eta)^{i} \lrp{\frac{g\lrp{\lrp{1+\tilde{\eta}}^{k-1}\lrp{1+\gamma}^j}}{\lrp{1+\eta}^{i+1}} - \Lambda_{Y}(1)}}.\]
	Notice that the term inside the expectation in the previous expression is a \(1\)-mean martingale, and the other term is bounded by choosing \(\theta =1\). Thus, (\ref{eq:3}), and hence, (\ref{eq:2}) is bounded by 
	\begin{align*}
	&\exp\lrp{ -\lrp{\frac{g\lrp{\lrp{1+\tilde{\eta}}^{k-1}\lrp{1+\gamma}^j}}{\lrp{1+\eta}} - \lrp{1+\eta}^{i}\Lambda_{Y}(1)}}.\numberthis \label{eq:int}
	\end{align*} 
	We now evaluate \(\Lambda_{Y}(1)\) to simplify the above bound. Observe that  \[\lrp{1+\eta}^i\Lambda_{Y}(1) \leq  \lrp{1+\eta}^i\log\lrp{1-\epsilon_2\tilde{\lambda}_1 - \epsilon_1\tilde{\lambda}_2 + \Exp{\abs{X_l-{m}}}\delta_1 + \Exp{\abs{\tilde{B}-f(\abs{X_l})}}\delta_2}.\]
	Using this in (\ref{eq:int}), probability in (\ref{eq:2}) is bounded by 
	\begin{align*}
	e^{ -\frac{g\lrp{\lrp{1+\tilde{\eta}}^{k-1}\lrp{1+\gamma}^j}}{\lrp{1+\eta}}} \lrp{1-\epsilon_2\tilde{\lambda}_1 - \epsilon_1\tilde{\lambda}_2 + \Exp{\abs{X_l-{m}}}\delta_1 + \Exp{\abs{\tilde{B}-f(\abs{X_l})}}\delta_2}^{\lrp{1+\eta}^i}.
	\end{align*}
	Using \(1+x\leq e^{x} \), the above expression is less than 
	\begin{align*}
	e^{ -\frac{g\lrp{\lrp{1+\tilde{\eta}}^{k-1}\lrp{1+\gamma}^j}}{\lrp{1+\eta}}} e^{-\epsilon_2\tilde{\lambda}_1{\lrp{1+\eta}^i} - \epsilon_1\tilde{\lambda}_2{\lrp{1+\eta}^i} + \Exp{\abs{X_l-{m}}}\delta_1{\lrp{1+\eta}^i} + \Exp{\abs{\tilde{B}-f(\abs{X_l})}}\delta_2{\lrp{1+\eta}^i}}.\numberthis \label{eq:int1}
	\end{align*}
	Choosing \(\delta_1 \) and \(\delta_2 \) as follows:
	\[ \delta_1 = \frac{\lrp{1+\eta}^{-i}}{f^{-1}\lrp{\tilde{B}} -\tilde{m}  }, \;\; \& \;\; \delta_2 = \frac{\lrp{1+\eta}^{-i}}{\tilde{B}-f\lrp{\abs{\tilde{m}}}}, \]
	and substituting in (\ref{eq:int1}), the probability in (\ref{eq:2}) can be bounded by 
	\begin{align*}
	\exp\lrp{ -\frac{g\lrp{\lrp{1+\tilde{\eta}}^{k-1}\lrp{1+\gamma}^j}}{\lrp{1+\eta}}} \exp\lrp{-\epsilon_2\tilde{\lambda}_1{\lrp{1+\eta}^i} - \epsilon_1\tilde{\lambda}_2{\lrp{1+\eta}^i}} C_2,\numberthis \label{eq:int2}
	\end{align*}
	where, 
	\[C_2 = \exp\lrp{\frac{\E{\mu_1}{\abs{X_l-{m}}}}{f^{-1}(\tilde{B}) -{m} } + \frac{\E{\mu_1}{\abs{\tilde{B}-f(\abs{X_l})}}}{\tilde{B}-f\lrp{\abs{{m}}}} }. \]
	Recall that \((\tilde{\lambda}_1, \tilde{\lambda}_2)\) is a corner point of the grid under consideration, and hence, \(\tilde{\lambda}_1\) is either \(l_1\delta_1\),  or \((l_1+1)\delta_1\) and similarly, \(\tilde{\lambda}_2\in\lrset{l_2\delta_2, (1+l_2)\delta_2 }\). The above expression is further upper-bounded by choosing \(\tilde{\lambda}_1 = l_1\delta_1\) and \(\tilde{\lambda}_2 = l_2\delta_2 \). Probability in (\ref{eq:2}) is bounded  with these substitutions in (\ref{eq:int2}) by 
	\begin{align*}
	C_2\exp\lrp{ -\frac{g\lrp{\lrp{1+\tilde{\eta}}^{k-1}\lrp{1+\gamma}^j}}{\lrp{1+\eta}}} \exp\lrp{-\frac{\epsilon_2  l_1}{f^{-1}\lrp{\tilde{B}}-{m}} - \frac{\epsilon_1 l_2}{\tilde{B}-f\lrp{\abs{{m}}}}}.
	\end{align*}
	
	Using this bound in (\ref{eq:bigSum}), (\ref{eq:1}) is bounded by  
	\begin{align*}
	&\sum\limits_{l_1,l_2} C_2\exp\lrp{ -\frac{g\lrp{\lrp{1+\tilde{\eta}}^{k-1}\lrp{1+\gamma}^j}}{\lrp{1+\eta}}} \exp\lrp{-\frac{\epsilon_2  l_1}{f^{-1}\lrp{\tilde{B}}-{m}} - \frac{\epsilon_1 l_2}{\tilde{B}-f\lrp{\abs{{m}}}}}.
	\end{align*}
	Since the summation in the expression is finite for \(l_1\) and \(l_2 \) ranging over all positive integers, on summing, we get the desired bound. 
\end{proof}

\subsection{Proof of Proposition  \ref{th:DoublePeeling}}\label{App:proof:th:DoublePeeling}

Recall that for \(\epsilon_2 > 0\), and  \(\gamma > 0\), \({m}=m(\mu_1)-\epsilon_2\), and  \(J_M = \frac{\log\lrp{T/\lrp{1+\tilde{\eta}}^{k-1}}}{\log\lrp{1+\gamma}}\) and for \(j\in\lrset{0,1,\cdots, J_M}\), the event \(D_j\) is defined as 
\[D_{j} = \lrset{\lrp{1+\tilde{\eta}}^{k-1} \lrp{1+\gamma}^{j}\leq T^{k}_{a} \leq \lrp{1+\tilde{\eta}}^{k-1}\lrp{1+\gamma}^{j+1} }.\] 
Recall that for \(t>0\), \(\eta > 0 \),  and \(i\geq 0\), we had  \( C_{i}(t) = \lrset{\lrp{1+\eta}^{i}\leq N_1\lrp{t} \leq \lrp{1+\eta}^{i+1} } .\)

Let \(\eta_{T} = \delta'_T \), where recall that \(\delta^{'}_t := \log\lrp{1+\frac{1}{\log \log t}} \), and define \(I_{M,j} = \frac{\log\lrp{\lrp{1+\tilde{\eta}}^{k-1}\lrp{1+\gamma}^{j+1} }}{\log\lrp{1+\eta_{T}}}. \) Using union bound and Lemma \ref{lem:BoundOnSet} the required probability is bounded by  
\[\sum\limits_{j=0}^{J_M} \sum\limits_{i=0}^{I_{M,j}} C_1 \exp\lrp{ -\frac{g\lrp{\lrp{1+\tilde{\eta}}^{k-1}\lrp{1+\gamma}^j}}{\lrp{1+\eta_{T}}}}, \]
which equals 
\[\sum\limits_{j=0}^{J_M} I_{M,j}C_1 \exp\lrp{ -\frac{g\lrp{\lrp{1+\tilde{\eta}}^{k-1}\lrp{1+\gamma}^j}}{\lrp{1+\eta_{T}}}},\numberthis \label{eq:Bound}\]
where \(C_1\) is as defined in Lemma \ref{lem:BoundOnSet}. Recall that \(g(t) = \lrp{1+\log\lrp{1+\frac{1}{\log\log t}}}^2\log(t) \). For \( t = \lrp{1+\tilde{\eta}}^{k-1} \lrp{1+\gamma}^{j} \), \(t\leq T\) and  \(\delta^{'}_t \geq \delta^{'}_T\). Substituting for the function \(g\lrp{\cdot}\), bounding the terms involving \( \delta^{'}_t \) by \(\delta^{'}_T\), and using that for all \(j\in [J_M]\) \(\eta_{j,k} \geq \eta_{J_M,k} = \log\lrp{1+\frac{1}{\log\log T}} \), the last expression in (\ref{eq:Bound}) is bounded by 
\[\frac{C_1 \lrp{1+\tilde{\eta}}^{-\lrp{k-1}\lrp{1+\delta^{'}_T}} }{\log\lrp{1+\delta'_T}} \sum\limits_{j=0}^{J_M} \frac{\log\lrp{\lrp{1+\tilde{\eta}}^{k-1}\lrp{1+\gamma}^{j+1}}}{\lrp{1+\gamma}^{j(1+\delta^{'}_T)}}.\]
Bounding the above expression by summing for \(j\) ranging over all positive integers, we get the desired bound.

\end{document}